\documentclass[a4paper]{article}
\usepackage{amsmath,amsthm,latexsym,amssymb, color, calligra,mathrsfs}
\usepackage{srcltx}
\usepackage[margin=20mm]{geometry}

\newtheorem{theorem}{Theorem}[section]

\newtheorem{lemma}[theorem]{Lemma}

\newtheorem{corollary}[theorem]{Corollary}
\newtheorem{remark}[theorem]{Remark}

\newtheorem{assumption}[theorem]{Assumption}

\begin{document}

\bibliographystyle{alpha}

\title{Stochastic Gradient Hamiltonian Monte Carlo for Non-Convex Learning \thanks{Both authors were supported by the NKFIH (National Research, Development and Innovation Office, Hungary) grant KH 126505 and the ``Lend\"ulet'' grant LP 2015-6 of the Hungarian Academy of Sciences. The authors thank Minh-Ngoc Tran for helpful discussions.} }
\author{Huy N. Chau \and Mikl\'os R\'asonyi}
\date{\today}
\maketitle
\begin{abstract}
	Stochastic Gradient Hamiltonian Monte Carlo (SGHMC) is a momentum version of stochastic gradient descent with properly injected Gaussian noise to find a global minimum. In this paper, non-asymptotic convergence analysis of SGHMC is given in the context of non-convex optimization, where subsampling techniques are used over an i.i.d dataset for gradient updates. Our results complement those of \cite{raginsky} and improve on those of \cite{gao}.
\end{abstract}



\section{Introduction}
Let $(\Omega,\mathcal{F},P)$ be a probability space where all the random objects of this paper will be defined. The expectation of a random variable $X$ with values in a Euclidean space will be denoted by $E[X]$.

We consider the following optimization problem
\begin{equation}\label{eq_prob}
F^*: = \min_{x \in \mathbb{R}^d} F(x),\mbox{ where } F(x):= E\left[ f(x,Z) \right] = \int_{\mathcal{Z}}{f(x,z) \mu(dz)} ,\ x\in\mathbb{R}^d 
\end{equation}
and $Z$ is a random element in some measurable space $\mathcal{Z}$ with an unknown probability law $\mu$. The function $x \mapsto f(x,z)$ is assumed continuously differentiable (for each $z$) but it can possibly be non-convex. Suppose that one has access to i.i.d samples $\mathbf{Z} = (Z_1,...,Z_n)$ drawn from $\mu$, where $n \in \mathbb{N}$ is fixed. Our goal is to compute an approximate minimizer $X^{\dagger}$ such that the \textit{population risk}
$$E[F(X^{\dagger})] - F^*$$
is minimized, where the expectation is taken with respect to the training data $\mathbf{Z}$ and additional randomness generating $X^{\dagger}$. 

Since the distribution of $Z_i, i \in \mathbb{N}$ is unknown, we consider the \textit{empirical risk minimization} problem
\begin{equation}\label{eq_prob_emp}
\min_{x \in \mathbb{R}^d} F_{\mathbf{z}}(x), \text{ where } F_{\mathbf{z}}(x):= \frac{1}{n} \sum_{i=1}^n f(x,z_i)
\end{equation}
using the dataset $\mathbf{z} := \{z_1,...,z_n\}$

Stochastic gradient algorithms based on Langevin Monte Carlo have gained more attention in recent years. Two popular algorithms are Stochastic Gradient Langevin Dynamics (SGLD) and Stochastic Gradient Hamiltonian Monte Carlo (SGHMC). First, we summarize the use of SGLD in optimization, 
as presented in \cite{raginsky}. Consider the overdamped Langevin stochastic differential equation
\begin{equation}\label{langevin}
dX_t = - \nabla F_{\mathbf{z}}(X_t) dt + \sqrt{2 \beta^{-1}} dB_t,
\end{equation}
where $(B_t)_{t \ge 0}$ is the standard Brownian motion in $\mathbb{R}^d$ and $\beta >0$ is the inverse temperature parameter. Under suitable assumptions on $f$, the SDE (\ref{langevin}) admits the Gibbs measure $\pi_{\mathbf{z}}(dx) \propto \exp(-\beta F_{\mathbf{z}}(x))$ as its unique invariant distribution. In addition, it is known that for sufficiently big $\beta$, the Gibbs distribution concentrates around global minimizers of $F_{\mathbf{z}}$. Therefore, one can use the value of $X_t$ from (\ref{langevin}), (or from its discretized counterpart SGLD), as an approximate solution to the empirical risk problem, provided that $t$ is large and temperature is low.

In this paper, we consider the underdamped (second-order) Langevin diffusion
\begin{eqnarray}
dV_t &=& - \gamma V_tdt - \nabla F_{\mathbf{z}}(X_t)dt + \sqrt{2 \gamma \beta^{-1}}dB_t, \label{eq_V}\\
dX_t &=& V_tdt \label{eq_X},
\end{eqnarray}
where $(X_t)_{t \ge 0}, (V_t)_{t \ge 0}$ model the position and the momentum of a particle moving in a field of force $F_{\mathbf{z}}$ with random force given by Gaussian noise. It is shown that under some suitable conditions for $F_{\mathbf{z}}$, the Markov process $(X,V)$ is ergodic and has a unique stationary distribution
$$\pi_{\mathbf{z}}(dx,dv) = \frac{1}{\Gamma_{\mathbf{z}}} \exp\left( -\beta\left( \frac{1}{2}\|v\|^2 + F_{\mathbf{z}}(x) \right)  \right) dx dv  $$
where $\Gamma_{\mathbf{z}}$ is the normalizing constant
$$\Gamma_{\mathbf{z}} = \left( \frac{2 \pi}{\beta} \right)^{d/2} \int_{\mathbb{R}^d} {e^{-\beta F_{\mathbf{z}}(x)}dx}. $$
It is easy to observe that the $x$-marginal distribution of $\pi_{\mathbf{z}}(dx,dv)$ is the invariant distribution $\pi_{\mathbf{z}}(dx)$ of (\ref{langevin}). We consider the first order Euler discretization of (\ref{eq_V}), (\ref{eq_X}), also called Stochastic Gradient Hamiltonian Monte Carlo (SGHMC), given as follows
\begin{eqnarray}
\overline{V}^{\lambda}_{k+1} &=& \overline{V}^{\lambda}_k - \lambda[\gamma \overline{V}^{\lambda}_k + \nabla F_{\mathbf{z}}(\overline{X}^{\lambda}_k)] + \sqrt{2 \gamma \beta^{-1} \lambda} \xi_{k+1}, \qquad \overline{V}^{\lambda}_0 = v_0,\label{eq_V_dis_aver}\\
\overline{X}^{\lambda}_{k+1} &=& \overline{X}^{\lambda}_k + \lambda \overline{V}^{\lambda}_{k}, \qquad \overline{X}^{\lambda}_0 = x_0, \label{eq_X_dis_aver}
\end{eqnarray}  
where $\lambda>0$ is a step size parameter and $(\xi_k)_{k \in \mathbb{N}}$ is a sequence of i.i.d standard Gaussian random vectors in $\mathbb{R}^d$. The initial condition $v_0,x_0$ may be random, but independent of $(\xi_k)_{k \in \mathbb{N}}$.

In certain contexts, the full knowledge of the gradient $F_{\mathbf{z}}$ is not available, however, using the dataset $\mathbf{z}$, one can construct its unbiased estimates. In what follows, we adopt the general setting given by \cite{raginsky}. Let $\mathcal{U}$ be a measurable space, and $g: \mathbb{R}^d \times \mathcal{U} \to \mathbb{R}^d$ such that for any $\mathbf{z} \in \mathcal{Z}^n$,
\begin{equation}\label{eq_unbiased}
E\left[g(x,U_{\mathbf{z}}) \right] = \nabla F_{\mathbf{z}}(x),  \forall x \in \mathbb{R}^d,
\end{equation}
where $U_{\mathbf{z}}$ is a random element in $\mathcal{U}$ with probability law $Q_{\mathbf{z}}$. Conditionally on $\mathbf{Z} = \mathbf{z}$, the SGHMC algorithm is defined by
\begin{eqnarray}
V^{\lambda}_{k+1} &=& V^{\lambda}_k - \lambda[\gamma V^{\lambda}_k + g(X^{\lambda}_k,U_{\mathbf{z},k})] + \sqrt{2 \gamma \beta^{-1} \lambda} \xi_{k+1}, \qquad V^{\lambda}_0 =v_0, \label{eq_V_dis_appr}\\
X^{\lambda}_{k+1} &=& X^{\lambda}_k + \lambda V^{\lambda}_k, \qquad X^{\lambda}_0 = x_0,\label{eq_X_dis_appr}
\end{eqnarray}
where $(U_{\mathbf{z},k})_{k \in \mathbb{N}}$ is a sequence of i.i.d. random elements in $\mathcal{U}$ with law $Q_{\mathbf{z}}$. We also assume from now on that $v_0, x_0, (U_{\mathbf{z},k})_{k \in \mathbb{N}}, (\xi_k)_{k \in \mathbb{N}}$ are independent.

Our ultimate goal is to find approximate global minimizers to the problem (\ref{eq_prob}). Let $X^{\dagger} := X^{\lambda}_k$ be the output of the algorithm (\ref{eq_V_dis_appr}),(\ref{eq_X_dis_appr}) after $k \in \mathbb{N}$ iterations, and $(\widehat{X}^*_{\mathbf{z}},\widehat{V}^*_{\mathbf{z}})$ be such that $\mathcal{L}(\widehat{X}^*_{\mathbf{z}},\widehat{V}^*_{\mathbf{z}}) = \pi_{\mathbf{z}}$. The excess risk is decomposed as follows, see also \cite{raginsky},
\begin{eqnarray}
E[F(X^{\dagger})] - F^* &=& \underbrace{\left( E[F(X^{\dagger})] - E[F(\widehat{X}^*_{\mathbf{z}})] \right)}_{\mathcal{T}_1} + \underbrace{\left( E[F(\widehat{X}^*_{\mathbf{z}})] - E[F_{\mathbf{Z}}(\widehat{X}^*_{\mathbf{Z}})] \right)}_{\mathcal{T}_2} \nonumber \\
&& + \underbrace{\left(E\left[ F_{\mathbf{Z}}(\widehat{X}^*_{\mathbf{Z}}) - F^* \right]\right) }_{\mathcal{T}_3}.\label{eq_decom_risk0}
\end{eqnarray}
The remaining part of the present paper is about finding bounds for these errors.  Section \ref{sec_main} summarizes technical conditions and the main results. Comparison of our contributions to previous studies is discussed in Section \ref{sec_related}. Proofs are given in Section \ref{sec_proof}.

\emph{Notation and conventions.} For $l\geq 1$, scalar product in $\mathbb{R}^{l}$
is denoted by $\langle \cdot,\cdot\rangle$. We use $\| \cdot \|$ to denote
the Euclidean norm (where the dimension of the space may vary). $\mathcal{B}(\mathbb{R}^{l})$ denotes the Borel $\sigma$- field of $\mathbb{R}^{l}$. For any $\mathbb{R}^{l}$-valued random variable $X$ and for any $1\leq p<\infty$, let us set
$\Vert X\Vert_p:=E^{1/p}\|X\|^p$. We denote by $L^p$ the set of $X$ with $\Vert X\Vert_p<\infty$.
The Wasserstein distance of order $p \in [1,\infty)$ between two probability measures $\mu$ and $\nu$ on $\mathcal{B}(\mathbb{R}^{l})$ is defined by
\begin{equation}\label{w_dist}
\mathcal{W}_p(\mu,\nu) = \left( \inf_{\pi \in \Pi(\mu,\nu)} \int_{\mathbb{R}^l} \Vert x-y\Vert^p
d\pi(x,y)  \right)^{1/p},
\end{equation}
where $\Pi(\mu,\nu)$ is the set of couplings of $(\mu, \nu)$, see e.g. \cite{v}. For two $\mathbb{R}^l$-valued random variables $X$ and $Y$, we denote $\mathfrak{W}_2(X,Y):= \mathcal{W}_2(\mathcal{L}(X),\mathcal{L}(Y))$, where $\mathcal{L}(X)$ is the law of $X$. We do not indicate $l$ in the notation and it may vary.

\section{Asumptions and main results}\label{sec_main}
The following conditions are required throughout the paper.
\begin{assumption}\label{as_f_bound}
	The function $f$ is continuously differentiable, takes non-negative values, and there are constants $A_0,B \ge 0$ such that for any $z \in \mathcal{Z}$,
	$$\|f(0,z)\| \le A_0, \qquad \|\nabla f(0,z)\| \le B.$$
\end{assumption}
\begin{assumption}\label{as_lip}
	There is $M>0$ such that, for each $z \in \mathcal{Z}$, 
	\begin{equation*}
	\| \nabla f(x_1,z) - \nabla f(x_2,z) \| \le M\|x_1 - x_2\|, \qquad \forall x_1,x_2 \in \mathbb{R}^d.
	\end{equation*}
\end{assumption}
\begin{assumption}[Dissipative]\label{as_dissip}
	There exist constants $m >0, b \ge 0$ such that
	\begin{equation*}
	\left\langle x, f(x,z) \right\rangle \ge m\|x\|^2 -b, \qquad \forall x \in \mathbb{R}^d, z \in \mathcal{Z}.
	\end{equation*}
\end{assumption}
\begin{assumption}\label{as_lip_g} For each 
	$u\in \mathcal{U}$, it holds that $\|g(0,u)\| \le B$ and
	\begin{equation*}
	\|g(x_1, u) - g(x_2, u)\| \le M \| x_1 - x_2 \|, \qquad \forall  x_1, x_2 \in \mathbb{R}^d.
	\end{equation*}
\end{assumption}
\begin{assumption}\label{as_variance}
	There exists a constant $\delta >0$ such that for every $\mathbf{z} \in \mathcal{Z}^n$,
	$$E\|g(x,U_{\mathbf{z}}) - \nabla F_{\mathbf{z}}(x) \|^2 \le 2\delta(M^2 \|x\|^2 + B^2). $$
\end{assumption}
\begin{assumption}\label{ass_exp_moment}
	The law $\mu_0$ of the initial state $(x_0,v_0)$ satisfies
	$$ \int_{\mathbb{R}^{2d}}{ e^{\mathcal{V}(x,v)} d\mu_0(x,v) } < \infty,$$
	where $\mathcal{V}$ is the Lyapunov function defined in (\ref{eq:lyapunov}) below.
\end{assumption}
\begin{remark}
	If the set of global minimizers is bounded, we can always redefine the function $f$ to be quadratic outside a compact set containing the origin while maintaining its minimizers. Hence, Assumption \ref{as_dissip} can be satisfied in practice. Assumption \ref{as_lip_g} means that the estimated gradient is also Lipschitz when using the same training dataset. For example, at each iteration of SGHMC, we may sample uniformly with replacement a random 
	minibatch of size $\ell$. Then we can choose $U_{\mathbf{z}} = (z_{I_1},...,z_{I_{\ell}})$ where $I_1,...,I_{\ell}$ are i.i.d random variables having distribution $\text{Uniform}(\{1,...,n\})$. The  gradient estimate is thus
	$$g(x,U_{\mathbf{z}}) = \frac{1}{\ell} \sum_{j=1}^{\ell} \nabla f(x, z_{I_j}),$$
	which is clearly unbiased and Assumption \ref{as_lip_g} will be satisfied whenever Assumptions \ref{as_lip} and \ref{as_f_bound} are in force. Assumption \ref{as_variance} controls the variance of the gradient estimate. 
\end{remark}

An auxiliary continuous time process is needed in the subsequent analysis. For a step size $\lambda >0$, denote by $B^{\lambda}_t := \frac{1}{\sqrt{\lambda}} B_{\lambda t}$ the scaled Brownian motion. Let $\widehat{V}(t,s,(v,x)), \widehat{X}(t,s,(v,x))$ be the solutions of
\begin{eqnarray}
\qquad d\widehat{V}(t,s,(v,x)) &=& - \lambda \left(  \gamma \widehat{V}(t,s,(v,x)) + \nabla F_{\mathbf{z}}(\widehat{X}(t,s,(v,x)))\right) dt + \sqrt{2 \gamma \lambda \beta^{-1}}dB^{\lambda}_t, \label{eq_V_hat}\\
\qquad d\widehat{X}(t,s,(v,x)) &=& \lambda \widehat{V}(t,s,(v,x))dt \label{eq_X_hat},
\end{eqnarray}
with initial condition $\widehat{V}_s = v, \widehat{X}_s = x$ where $v,x$ may be random but independent of $(B^{\lambda}_t)_{t \ge 0}$.

Our first result tracks the discrepancy between the SGHMC algorithm (\ref{eq_V_dis_appr}), (\ref{eq_X_dis_appr}) and the auxiliary processes (\ref{eq_V_hat}), (\ref{eq_X_hat}). 
\begin{theorem}\label{thm_algo_aver}
	Let $1 \le p \le 2$. There exists a constant $\tilde{C} > 0$ such that for all $k \in \mathbb{N}$,
	\begin{equation}
	\mathfrak{W}_p((V^{\lambda}_k, X^{\lambda}_k) , (\widehat{V}(k,0,(v_0,x_0)), \widehat{X}(k,0,(v_0,x_0)))) \le \tilde{C} (\lambda^{1/(2p)} + \delta^{1/(2p)}).
	\end{equation}
\end{theorem}
\begin{proof}
	The proof of this theorem is given in Section \ref{proof_thm_algo_aver}.
\end{proof}
The following is the main result of the paper.
\begin{theorem}\label{thm_main}
	Let $1 < p \le 2$. Suppose that the SGHMC iterates $(V^{\lambda}_k,X^{\lambda}_k)$ are defined by (\ref{eq_V_dis_appr}), (\ref{eq_X_dis_appr}). The expected population risk can be bounded as
	$$E[F(X^{\lambda}_k)] - F^* \le \mathcal{B}_1 + \mathcal{B}_2 + \mathcal{B}_3,$$
	where
	\begin{align*}
	\mathcal{B}_1 &:=(M \sigma + B) \left( \tilde{C} (\lambda^{1/(2p)} + \delta^{1/(2p)}) +  C_*  (\mathcal{W}_{\rho}(\mu_0, \pi_{\mathbf{z}}))^{1/p} \exp(- c_* k \lambda )  \right),\\
	\mathcal{B}_2 &:=  \frac{4\beta c_{LS}}{n}\left( \frac{M^2}{m}(b+d/\beta) + B^2 \right) ,\\
	\mathcal{B}_3 &:= \frac{d}{2\beta} \log\left( \frac{eM}{m}\left( \frac{b\beta}{d} + 1 \right) \right), 
	\end{align*}
	where $\tilde{C}, C_*, c_*, c_{LS}$ are appropriate constants and $\mathcal{W}_{\rho}$ is the metric defined in (\ref{eq_dist_aux}) below.
\end{theorem}
\begin{proof}
	The proof of this theorem is given in Section \ref{proof_thm_main}.
\end{proof}
\begin{corollary}\label{cor}
	Let $1 \le p \le 2, \varepsilon > 0$ We have 
	$$\mathcal{W}_p(\mathcal{L}(X_k), \pi_{\mathbf{z}}) \le \varepsilon$$
	whenever
	$$ (\lambda^{1/(2p)} + \delta^{1/(2p)}) \le \frac{1}{2\tilde{C}}\varepsilon, \qquad k \ge \frac{(2\tilde{C})^{2p}}{c_*} \frac{1}{\varepsilon^{2p}} \log\left( \frac{C_*  (\mathcal{W}_{\rho}(\mu_0, \pi_{\mathbf{z}}))^{1/p}}{\varepsilon} \right) .$$
\end{corollary}
\begin{proof} From the proof of Theorem \ref{thm_main}, or more precisely from (\ref{sampling_err}), we need to choose $\lambda$ and $k$ such that
	$$\tilde{C} (\lambda^{1/(2p)} + \delta^{1/(2p)}) +  C_*  (\mathcal{W}_{\rho}(\mu_0, \pi_{\mathbf{z}}))^{1/p} \le \varepsilon.$$
	First, we choose $\lambda$ and $\delta$ so that $\tilde{C}(\lambda^{1/(2p)} + \delta^{1/(2p)}) < \varepsilon/2$ and then $$C_*  (\mathcal{W}_{\rho}(\mu_0, \pi_{\mathbf{z}}))^{1/p} \exp(- c_* k \lambda ) \le \varepsilon/2$$
	will hold for $k$ large enough.
\end{proof}

\section{Related work and our contributions}\label{sec_related}
Non-asymptotic convergence rate Langevin dynamics based algorithms for approximate sampling log-concave distributions are intensively studied in recent years. For example, overdamped Langevin dynamics are discussed in  \cite{welling2011bayesian}, \cite{dalalyan2017theoretical}, \cite{durmus2016high}, \cite{dalalyan2017user}, \cite{dm} and others. Recently, \cite{six} treats the case of non-i.i.d. data streams with a certain mixing property. Underdamped Langevin dynamics are examined in \cite{chen2014stochastic}, \cite{neal2011mcmc}, \cite{cheng2017underdamped}, etc.   Further analysis on HMC are discussed on \cite{betancourt2017geometric}, \cite{betancourt2017conceptual}. Subsampling methods are applied to speed up HMC for large datasets, see \cite{dang2017hamiltonian}, \cite{quiroz2018speeding}. 

The use of momentum to accelerate optimization methods are discussed intensively in literature, for example \cite{attouch2016rate}. In particular, performance of SGHMC is experimentally proved better than SGLD in many applications, see \cite{chen2015convergence}, \cite{chen2014stochastic}. An important advantage of the underdamped SDE is that convergence to its stationary distribution is faster than that of the overdamped SDE in the $2$-Wasserstein distance, as shown in \cite{ear2}. 

Finding an approximate minimizer is similar to sampling distributions concentrate around the true minimizer. This well known connection gives rise to the study of simulated annealing algorithms, see \cite{hwang1980laplace},  \cite{gidas1985nonstationary}, \cite{hajek1985tutorial}, \cite{chiang1987diffusion}, \cite{holley1989asymptotics}, \cite{gelfand1991recursive}, \cite{gelfand1993metropolis}. Recently, there are many studies further investigate this connection by means of non asymptotic convergence of Langevin based algorithms and in stochastic non-convex optimization and
large-scale data analysis, \cite{chen2016bridging}, \cite{dalalyan2017further}.

Relaxing convexity is a more challenging issue.  In \cite{cheng2018sharp}, the problem of sampling from a target distribution $\exp(-F(x))$ where $F$ is L-smooth everywhere and $m$-strongly convex outside a ball of finite radius is considered. They provide upper bounds for the number of steps to be within a given precision level $\varepsilon$ of the 1-Wasserstein distance between the HMC algorithm and the equilibrium distribution. In a similar setting,  \cite{majka2018non} obtains bounds in both the $\mathcal{W}_1$ and $\mathcal{W}_2$ distances for overdamped Langevin dynamics with stochastic gradients. \cite{xu2018global} studies the convergence of the SGLD algorithm and the variance reduced SGLD to global minima of nonconvex functions satisfying the dissipativity condition.

Our work continues these lines of research, the most similar setting to ours is the recent paper \cite{gao}. We summarize our contributions below: 
\begin{itemize}
	\item Diffusion approximation. In Lemma 10 of \cite{gao}, the upper bound for the 2-Wasserstein distance between the SGHMC algorithm at step $k$ and underdamped SDE at time $t = k \lambda$ is (up to constants) given by
	$$ (\delta^{1/4} + \lambda^{1/4}) \sqrt{k \lambda} \sqrt{\log(k\lambda)},$$
	which depends on the number of iteration $k$. Therefore obtaining a precision $\varepsilon$ requires a careful choice of $k, \lambda$ and even $k\lambda$. By introducing the auxiliary SDEs (\ref{eq_V_hat}), (\ref{eq_X_hat}), we are able to achieve the rate 
	\begin{equation*}
	(\delta^{1/4} + \lambda^{1/4}),
	\end{equation*} 
	see Theorem \ref{thm_algo_aver} for the case $p=2$. This upper bound is better 
	in the number of iterations and hence, improves Lemma 10  of \cite{gao}. Our analysis for variance of the algorithm is also different. The iteration does not accumulate mean squared errors, as the number of step goes to infinity.  
	\item Our proof for Theorem \ref{thm_algo_aver} is relatively simple and we do not need to adopt the techniques of \cite{raginsky} which involve heavy functional analysis, e.g. the weighted Csisz\'ar - Kullback - Pinsker inequalities in \cite{bolley2005weighted} is not needed.  
	\item If we consider the $p$-Wasserstein distance for $1<p\le 2$, in particular, when $p \to 1$, Theorem \ref{thm_main} gives tighter bounds, compared to Theorem 2 of \cite{gao}.
	\item Dependence structure of the dataset in the sampling mechanism, can be \textit{arbitrary}, see the proof of Theorem \ref{thm_algo_aver}. The i.i.d  assumption on dataset is used only for the generalization error. We could also  incorporate non-i.i.d data in our analysis, see Remark \ref{re_extend_T2}, but this is left for further research.
\end{itemize} 



\section{Proofs}\label{sec_proof}
\subsection{A contraction result}
In this section, we recall a contraction result of \cite{ear2}. First, it should be noticed that the constant $u$ and the function $U$ in their paper are $\beta^{-1}$ and $\beta F_{\mathbf{z}}$ in the present paper, respectively. Here, the subscript $c$ stands for ``contraction". Using the upper bound of Lemma \ref{lem_quad} for $f$ below, there exist constants $\lambda_c \in \left(0, \min\{1/4, m/(M + 2B + \gamma^2/2)\} \right)$ small enough and $A_c \ge \beta/2(b + 2B + A_0)$
such that
\begin{equation}\label{eq:drift}
\left\langle x ,\nabla F_{\mathbf{z}} (x) \right\rangle  \ge m\|x\|^2 - b \ge 2 \lambda_c (  F_{\mathbf{z}}(x) + \gamma^2 \|x\|^2/4  ) - 2A_c/\beta.
\end{equation}
Therefore, Assumption 2.1 of \cite{ear2} is satisfied, noting that $L_c:= \beta M$ and 
$$\|\nabla F_{\mathbf{z}} (x) - \nabla F_{\mathbf{z}}(y)\| \le \beta^{-1} L_c\|x -y\|.$$
We define the Lyapunov function
\begin{equation}\label{eq:lyapunov}
\mathcal{V}(x,v) = \beta F_{\mathbf{z}}(x) + \frac{\beta}{4} \gamma^2 \left( \|x + \gamma^{-1}v \|^2 + \| \gamma^{-1} v\|^2 - \lambda_c \|x\|^2 \right),
\end{equation}
For any $(x_1, v_1), (x_2, v_2) \in \mathbb{R}^{2d}$, we set 
\begin{eqnarray}
r((x_1,v_2), (x_2,v_2)) &=& \alpha_c \|x_1 - x_2 \| + \| x_1 - x_2 + \gamma^{-1}(v_1 - v_2) \|, \label{eq:r}\\
\rho((x_1, v_1),(x_2,v_2)) &=& h(r((x_1, v_1),(x_2,v_2))) \left( 1 + \varepsilon_c \mathcal{V}(x_1,v_1) + \varepsilon_c \mathcal{V}(x_2,v_2) \right), \label{eq:rho}
\end{eqnarray} 
where $\alpha_c, \varepsilon_c >0$ are suitable positive constants to be fixed later and $h:[0, \infty) \to 
[0, \infty)$ is continuous, non-decreasing concave function such that $h(0) = 0$, $h$ is $C^2$ on $(0,R_1)$ for some constant $R_1 > 0$ with right-sided derivative $h'_+(0) = 1$ and left-sided derivative $h'_-(R_1) > 0$ and $h$ is constant on $[R_1, \infty)$. For any two probability measures $\mu, \nu$ on $\mathbb{R}^{2d}$, we define
\begin{equation}\label{eq_dist_aux}
\mathcal{W}_{\rho}(\mu, \nu): = \inf_{(X_1,V_1) \sim \mu, (X_2,V_2) \sim \nu} E\left[  \rho((X_1,V_1),(X_2,V_2))\right]. 
\end{equation}
Note that $\rho$ and $\mathcal{W}_{\rho}$ are semimetrics but not necessarily metrics. A result from \cite{ear2} is recalled below.

For a probability measure $\mu$ on $\mathcal{B}(\mathbb{R}^{2d})$, we denote by $\mu p_t$ the law of $(V_t, X_t)$ when $\mathcal{L}(V_0,X_0) = \mu$.
\begin{theorem}\label{thm_contraction}
	There exists a continuous non-decreasing concave function $h$ with $h(0) = 0$ such that for all probability measures $\mu, \nu$ on $\mathbb{R}^{2d}$, and $1 \le p \le 2$, we have
	\begin{equation}
	\mathcal{W}_p(\mu p_t, \nu p_t) \le C_*  \left( \mathcal{W}_{\rho}(\mu, \nu)\right)^{1/p}  \exp(- c_* t ), \qquad \forall t \ge 0,
	\end{equation}
	where the following relations hold:
	\begin{eqnarray*}
		c_* &=& \frac{\gamma}{384p} \min\{ \lambda_c M \gamma^{-2}, \Lambda_c^{1/2} e^{-\Lambda_c}M \gamma^{-2}, \Lambda_c^{1/2} e^{-\Lambda_c}  \},\\
		C_* &=& 2^{1/p}e^{2/p + \Lambda_c/p} \frac{1 + \gamma}{\min \{1, \alpha_c\}} \left(  \max\left\lbrace 1, 4\frac{\max\{1,R^{p-2}_1\}}{\min\{1,R_1\}}(1 + 2 \alpha_c + 2 \alpha_c^2)(d + A_c) \beta^{-1} \gamma^{-1} c_*^{-1} \right\rbrace  \right)^{1/p}  , \\
		\Lambda_c &=& \frac{12}{5}(1 + 2 \alpha_c + 2 \alpha_c^2)(d + A_c) M \gamma^{-2} \lambda_c^{-1}(1-2\lambda_c)^{-1},\\
		\alpha_c &=& (1 + \Lambda_c^{-1})M \gamma^{-2} >0 , \\
		\varepsilon_c &=& 4 \gamma^{-1} c_*/(d+A_c) >0,\\
		R_1 &=& 4 \cdot (6/5)^{1/2} (1+2 \alpha_c + 2 \alpha_c^2)^{1/2} (d + A_c)^{1/2} \beta^{-1/2} \gamma^{-1}(\lambda_c - 2 \lambda_c^2)^{-1/2}.
	\end{eqnarray*}
	The function $h$ is constant on $[R_1, \infty)$, $C^2$ on $(0,R_1)$ with 
	\begin{eqnarray*}
		f(r) &=& \int_0^{r \wedge R_1}{ \varphi(s) g(s) ds },\\
		\varphi(s)&=& \exp \left( -(1 + \eta_c)L_cs^2/8 - \gamma^2 \beta \varepsilon_c \max\{1,(2\alpha_c)^{-1}\} s^2/2  \right)   ,\\
		g(s) &=& 1- \frac{9}{4}c_*\gamma \beta \int_0^{r} {\Phi(s)\varphi(s)^{-1}ds}, \qquad \Phi(s) = \int_0^s{\varphi(x)dx}
	\end{eqnarray*}
	and $\eta_c $ satisfies $\alpha_c = (1+\eta_c)L_c \beta^{-1} \gamma^{-2}$.
\end{theorem}
\begin{proof}
	From (5.15) of \cite{ear2}, we get
	$$\|(x_1,v_1) - (x_2,v_2) \|^p \le \frac{(1+\gamma)^p}{\min\{1, \alpha^{p}_c\}} r((x_1,v_1),(x_2,v_2))^{p}.$$
	Furthermore, from the proof of Corollary 2.6 of \cite{ear2}, if  $r:=r((x_1,v_1),(x_2,v_2)) \le \min\{1,R_1\}$,
	$$ r^2 \le r^p \le r \le 2e^{2+\Lambda_c} \rho((x_1,v_1),(x_2,v_2))$$
	and if $r \ge \min\{1,R_1\} $ then
	$$r^p \le \max\{1,R^{p-2}_1\} r^2 \le \frac{\max\{1,R^{p-2}_1\}}{\min\{1,R_1\}}  8e^{2 + \Lambda_c}(1 + 2 \alpha_c + 2\alpha^2_c)(d+A_c)\beta^{-1}\gamma^{-1}c_*^{-1} \rho((x_1,v_1),(x_2,v_2)).$$
	These bounds and Theorem 2.3 of \cite{ear2} imply that
	$$\mathcal{W}_p(\mu p_t, \nu p_t) \le C_*  \left( \mathcal{W}_{\rho}(\mu, \nu)\right)^{1/p} \exp(- c_* t ).$$
	The proof is complete.
\end{proof}
It should be emphasized that $(\widehat{V}(t,0,(v_0,x_0)), \widehat{X}(t,0,(v_0,x_0)))= (V_{\lambda t}, X_{\lambda t})$, and consequently, $(\widehat{V}(t,0,(v_0,x_0)), \widehat{X}(t,0,(v_0,x_0)))$ contracts at the rate $\exp(- c_* \lambda t)$.  
\subsection{Proof of Theorem \ref{thm_algo_aver}}\label{proof_thm_algo_aver}
Here, we summarize our approach. For a given step size $\lambda >0$, we divide the time axis into intervals of length $T = \lfloor 1/\lambda\rfloor$. For each time step $k \in [nT,(n+1)T], n \in \mathbb{N}$, we compare the SGHMC to the version with exact gradients relying on the Doob inequality, and then compare the later to the auxiliary continuous-time diffusion $(\widehat{V}(k,0,(v_0,x_0)), \widehat{X}(k,0,(v_0,x_0)))$ with the scaled Brownian motion. At this stage we reply on the contraction result from \cite{ear2} and uniform boundedness of the Langevin diffusion and its discrete time versions. Since the auxiliary dynamics evolves slower than the original Langevin dynamics, or more precisely at the same speed as that of the SGHCM, our upper bounds do not accumulate errors and are independent from the number of iterations. 
\begin{proof}
	For each $k \in \mathbb{N}$, we define
	$$ \mathcal{H}_k: = \sigma(U_{\mathbf{z},i}, 1 \le i \le k) \vee \sigma(\xi_j, j \in \mathbb{N}).$$
	Let $\tilde{v}, \tilde{x}$ be $\mathbb{R}^d$-valued random variables satisfying Assumption \ref{ass_exp_moment}. For $0 \le i \le j$, we recursively define  $\tilde{V}^{\lambda}(i,i,(\tilde{v},\tilde{x})) := \tilde{v}$, $\tilde{X}^{\lambda}(i,i,(\tilde{v},\tilde{x})) := \tilde{x}$ and
	\begin{eqnarray}
	\tilde{V}^{\lambda}(j+1,i,(\tilde{v},\tilde{x})) &=& \tilde{V}^{\lambda}(j,i,(\tilde{v},\tilde{x})) - \lambda[\gamma \tilde{V}^{\lambda}(j,i,(\tilde{v},\tilde{x})) + \nabla F_{\mathbf{z}}(\tilde{X}^{\lambda}(j,i,(\tilde{v},\tilde{x})))] \nonumber \\
	&&  + \sqrt{2 \gamma \beta^{-1} \lambda} \xi_{j+1}, \label{eq_V_dis_tilde}\\
	\tilde{X}^{\lambda}(j+1,i,(\tilde{v},\tilde{x})) &=& \tilde{X}^{\lambda}(j,i,(\tilde{v},\tilde{x})) + \lambda \tilde{V}^{\lambda}(j,i,(\tilde{v},\tilde{x})). \label{eq_X_dis_tilde}
	\end{eqnarray}
	Let $T: = \lfloor 1/\lambda\rfloor$. For each $n \in \mathbb{N}$, and for each $nT \le k < (n+1)T$, we set
	\begin{eqnarray}\label{pro_tilde}
	\tilde{V}^{\lambda}_k: = \tilde{V}^{\lambda}(k,nT,(V^{\lambda}_{nT},X^{\lambda}_{nT})), \qquad \tilde{X}^{\lambda}_k:= \tilde{X}^{\lambda}(k,nT,(V^{\lambda}_{nT},X^{\lambda}_{nT})).
	\end{eqnarray}
	For each $n \in \mathbb{N}$, it holds by definition that $V^{\lambda}_{nT} = \tilde{V}^{\lambda}_{nT}$ and the triangle inequality implies for $nT \le k < (n+1)T$,
	\begin{eqnarray*}
		\|V^{\lambda}_k- \tilde{V}^{\lambda}_k \|  \le \lambda \left\|  \sum_{i = nT}^{k-1} \left( g(X^{\lambda}_{i},U_{\mathbf{z},i})  - \nabla F_{\mathbf{z}}(\tilde{X}^{\lambda}_i)\right) \right\|	
	\end{eqnarray*}
	and
	\begin{equation} \label{in_X_tri}
	\left\| X^{\lambda}_k - \tilde{X}^{\lambda}_k \right\|  
	\le \lambda \sum_{i = nT}^{k-1}  \left\| V^{\lambda}_i - \tilde{V}^{\lambda}_i  \right\|.
	\end{equation}
	Denote $g_{k,nT}(x):= E\left[ g(x,U_{\mathbf{z},k})| \mathcal{H}_{nT} \right], x \in \mathbb{R}^d$. By Assumption \ref{as_lip_g}, the estimation continues as follows
	\begin{eqnarray}
	\|V^{\lambda}_k- \tilde{V}^{\lambda}_k \|  &\le& \lambda \left\|  \sum_{i = nT}^{k-1} \left( g(X^{\lambda}_{i},U_{\mathbf{z},i})  - \nabla F_{\mathbf{z}}(\tilde{X}^{\lambda}_i)\right) \right\|  \le \lambda \sum_{i = nT}^{k-1} \left\| g(X^{\lambda}_{i},U_{\mathbf{z},i})  - g(\tilde{X}^{\lambda}_i, U_{\mathbf{z},i})  \right\| \nonumber \\
	&&  + \lambda  \left\| \sum_{i = nT}^{k-1}  g(\tilde{X}^{\lambda}_i, U_{\mathbf{z},i}) - g_{i,nT}(\tilde{X}^{\lambda}_i) \right\|  +  \lambda \sum_{i = nT}^{k-1} \left\| g_{i,nT}(\tilde{X}^{\lambda}_i) - \nabla F_{\mathbf{z}}(\tilde{X}^{\lambda}_i) \right\| \nonumber \\
	&\le& \lambda M \sum_{i = nT}^{k-1} \| X^{\lambda}_{i} - \tilde{X}^{\lambda}_{i} \| +  \lambda \max_{nT \le m < (n+1)T}  \left\| \sum_{i = nT}^{m}  g(\tilde{X}^{\lambda}_i, U_{\mathbf{z},i}) - g_{i,nT}(\tilde{X}^{\lambda}_i) \right\| \nonumber \\
	&&  +  \lambda \sum_{i = nT}^{(n+1)T-1} \left\| g_{i,nT}(\tilde{X}^{\lambda}_i) - \nabla F_{\mathbf{z}}(\tilde{X}^{\lambda}_i) \right\|. \label{eq_long}
	\end{eqnarray} 
	Using (\ref{in_X_tri}), one obtains
	\begin{eqnarray}
	\sum_{i = nT}^{k-1} \| X^{\lambda}_{i} - \tilde{X}^{\lambda}_{i} \| &\le& \lambda T\| V^{\lambda}_{nT} - \tilde{V}^{\lambda}_{nT} \| +...+ \lambda T\|V^{\lambda}_{k-1} - \tilde{V}^{\lambda}_{k-1}\| \nonumber \\ 
	&\le& \sum_{i = nT}^{k-1} \| V^{\lambda}_i - \tilde{V}^{\lambda}_i \|,
	\end{eqnarray}
	noting that $T\lambda \le 1.$ Therefore, the estimation in (\ref{eq_long}) continues as
	\begin{eqnarray*}
		\|V^{\lambda}_k- \tilde{V}^{\lambda}_k \|  &\le& \lambda M \sum_{i = nT}^{k-1} \| V^{\lambda}_i - \tilde{V}^{\lambda}_i \| +  \lambda \max_{nT \le m < (n+1)T}  \left\| \sum_{i = nT}^{m}  g(\tilde{X}^{\lambda}_i, U_{\mathbf{z},i}) - g_{i,nT}(\tilde{X}^{\lambda}_i) \right\| \\
		&& +  \lambda \sum_{i = nT}^{(n+1)T-1} \left\| g_{i,nT}(\tilde{X}^{\lambda}_i) - \nabla F_{\mathbf{z}}(\tilde{X}^{\lambda}_i) \right\|.
	\end{eqnarray*}
	Applying the discrete-time version of Gr\"onwall's lemma and taking squares, noting also that $(x + y)^2 \le 2 (x^2 + y^2), x, y \in \mathbb{R}$ yield
	\begin{eqnarray*}
		&&\|V^{\lambda}_k- \tilde{V}^{\lambda}_k \|^2 \le 2 \lambda^2 e^{2MT\lambda} \left[ \max_{nT \le m < (n+1)T}  \left\| \sum_{i = nT}^{m}  g(\tilde{X}^{\lambda}_i, U_{\mathbf{z},i}) - g_{i,nT}(\tilde{X}^{\lambda}_i) \right\|^2 + \Xi_n^2  \right],
	\end{eqnarray*}
	where \begin{equation}\label{eq_defi_Xi}
	\Xi_n := \sum_{i = nT}^{(n+1)T-1} \left\| g_{i,nT}(\tilde{X}^{\lambda}_i) - \nabla F_{\mathbf{z}}(\tilde{X}^{\lambda}_i) \right\|.\end{equation}
	Taking conditional expectation with respect to $\mathcal{H}_{nT}$, the estimation becomes
	\begin{eqnarray*}
		E\left[ \left.  \|V^{\lambda}_k- \tilde{V}^{\lambda}_k \|^2 \right|\mathcal{H}_{nT} \right]  &\le& 2 \lambda^2 e^{2M} E\left[ \left.  \max_{nT \le m < (n+1)T}  \left\| \sum_{i = nT}^{m}  g(\tilde{X}^{\lambda}_i, U_{\mathbf{z},i}) - g_{i,nT}(\tilde{X}^{\lambda}_i) \right\|^2 \right| \mathcal{H}_{nT} \right] \\
		&+& 2\lambda^2 e^{2M} E [\Xi^2_n|\mathcal{H}_{nT}].
	\end{eqnarray*}
	Since the random variables $U_{\mathbf{z},i}$ are independent, the sequence of random variables $g(\tilde{X}^{\lambda}_i, U_{\mathbf{z},i}) - g_{i,nT}(\tilde{X}^{\lambda}_i)$, $nT\leq i<(n+1)T$ are 
	independent conditionally on $\mathcal{H}_{nT}$, noting that $\tilde{X}^{\lambda}_i$ is measurable with respect to $\mathcal{H}_{nT}$. In addition, they have zero mean by the tower property of conditional expectation.  By Assumption \ref{as_lip_g},
	$$\|g(x,u)\| \le M\|x\| + B$$
	and thus
	\begin{eqnarray}\label{eq_var_contr}E \left[ \|g(\tilde{X}^{\lambda}_i, U_{\mathbf{z},i})\|^2 | \mathcal{H}_{nT} \right] \le 2M^2 E\left[  \| \tilde{X}^{\lambda}_i \|^2\right]  + 2B^2. 
	\end{eqnarray}
	by the independence of $U_{\mathbf{z},i}, i > nT$  from $\mathcal{H}_{nT}$. 
	Doob's inequality and (\ref{eq_var_contr}) imply
	\begin{eqnarray*}
		E\left[ \left.  \max_{nT \le m < (n+1)T}  \left\| \sum_{i = nT}^{m}  g(\tilde{X}^{\lambda}_i, U_{\mathbf{z},i}) - g_{i,nT}(\tilde{X}^{\lambda}_i) \right\|^2 \right| \mathcal{H}_{nT}  \right] \le 8M^2 \sum_{i=nT}^{(n+1)T-1} E\left[  \| \tilde{X}^{\lambda}_i \|^2\right]  + 8B^2T.
	\end{eqnarray*}
	Taking one more expectation and using Lemma \ref{lem_uniform2} give
	\begin{eqnarray*}
		E\left[  \max_{nT \le m < (n+1)T}  \left\| \sum_{i = nT}^{m}  g(\tilde{X}^{\lambda}_i, U_{\mathbf{z},i}) - g_{i,nT}(\tilde{X}^{\lambda}_i) \right\|^2 \right] &\le& 8M^2 \sum_{i=nT}^{(n+1)T-1} E \left[ \| \tilde{X}^{\lambda}_i \|^2\right]  + 8B^2T\\
		&\le& (8M^2C^a_x + 8B^2)T.
	\end{eqnarray*}
	By Lemma \ref{lem_Xi}, we have $E [\Xi^2_n] < 2T^2\delta(M^2 C^a_x + B^2)$, and therefore,
	\begin{equation}\label{algo_tilde}
	E^{1/2}\left[ \|V^{\lambda}_k - \tilde{V}^{\lambda}_k \|^2\right]  \le c_2 \sqrt{\lambda} +  c_3 \sqrt{\delta}
	\end{equation}
	where we define 
	$$c_2=4e^M\sqrt{(M^2C^a_x + B^2)}, \qquad c_3= 2e^M\sqrt{M^2C^a_x + B^2}.$$
	Consequently, we have from (\ref{in_X_tri})
	\begin{eqnarray}
	E^{1/2} \left[ \left\| X^{\lambda}_k - \tilde{X}^{\lambda}_k \right\| ^2\right]  
	&\le& \lambda   \sum_{i = nT}^{k-1} E^{1/2}\left[  \left\| V^{\lambda}_i - \tilde{V}^{\lambda}_i  \right\|^2 \right]  \le \lambda T (c_2 \sqrt{\lambda} + c_3 \sqrt{\delta}) \nonumber \\
	&\le&  c_2 \sqrt{\lambda} + c_3 \sqrt{\delta}. \label{algo_tilde_2}
	\end{eqnarray}
	Let $\tilde{V}^{int}$ and $\tilde{X}^{int}$  be the continuous-time interpolation of $\tilde{V}^{\lambda}_k$, and of $\tilde{X}^{\lambda}_k$ on $[nT, (n+1)T)$, respectively,
	\begin{eqnarray}
	d\tilde{V}^{int}_t &=& - \lambda \gamma \tilde{V}^{int}_{\lfloor t \rfloor} dt - \lambda \nabla F_{\mathbf{z}}(\tilde{X}^{int}_{\lfloor t \rfloor})\, dt + \sqrt{2\gamma \lambda \beta^{-1}} dB^{\lambda}_t , \label{proc_V_int}\\
	d\tilde{X}^{int}_t &=& \lambda \tilde{V}^{int}_{\lfloor t \rfloor } dt, \label{proc_X_int}
	\end{eqnarray} 
	with the initial conditions $\tilde{V}^{int}_{nT} = \tilde{V}_{nT} = V^{\lambda}_{nT}$ and $\tilde{X}^{int}_{nT} = \tilde{X}_{nT} = X^{\lambda}_{nT}$.
	For each $n \in \mathbb{N}$ and for $nT \le t < (n+1)T$, define also
	\begin{equation}\label{proc_hat}
	\widehat{V}_t = \widehat{V}(t,nT,(V^{\lambda}_{nT},X^{\lambda}_{nT})), \qquad \widehat{X}_t = \widehat{X}(t,nT,(V^{\lambda}_{nT},X^{\lambda}_{nT})),
	\end{equation}
	where the dynamics of $\widehat{V}, \widehat{X}$ are given in (\ref{eq_V_hat}), (\ref{eq_X_hat}). In this way, the processes $(\widehat{V}_t)_{t \ge 0}, (\widehat{X}_t)_{t \ge 0}$ are right continuous with left limits. From Lemma \ref{lem_tilde_hat}, we obtain for $nT \le t < (n+1)T$
	\begin{equation}\label{tilde_hat}
	E^{1/2}\left[ \|\tilde{V}^{int}_t - \widehat{V}_t  \|^2\right]  \le c_7 \sqrt{\lambda}, \qquad E^{1/2}\left[ \|\tilde{X}^{int}_t - \widehat{X}_t \|^2\right] \le c_7 \sqrt{\lambda}.
	\end{equation}
	Combining (\ref{algo_tilde}), (\ref{algo_tilde_2}) and (\ref{tilde_hat}) gives
	\begin{equation}\label{algo_hat}
	E^{1/2}\left[ \| V^{\lambda}_k - \widehat{V}_k \|^2\right]  \le (c_2 + c_7) \sqrt{\lambda} + c_3 \sqrt{\delta}, \qquad E^{1/2}\left[ \| X^{\lambda}_k - \widehat{X}_k \|^2 \right] \le (c_2 + c_7) \sqrt{\lambda} + c_3 \sqrt{\delta}.
	\end{equation}
	Define $\widehat{A}_t = (\widehat{V}_t, \widehat{X}_t)$ and $\widehat{B}(t,s,(v_s,x_s)) = (\widehat{V}(t,s,(v_s,x_s)), \widehat{X}(t,s,(v_s,x_s)))$ for $s \le t$ and $v_s, x_s$ are $\mathbb{R}^d$-valued random variables. The triangle inequality and Theorem \ref{thm_contraction} imply that for $nT \le t < (n+1)T$, and for $1 \le p \le 2$,
	\begin{eqnarray}
	&&\mathfrak{W}_p(\widehat{A}_t, \widehat{B}(t,0,(v_0,x_0))) \nonumber \\
	&\le& \sum_{i = 1}^n \mathfrak{W}_p( \widehat{B}(t, iT, (V^{\lambda}_{iT}, X^{\lambda}_{iT})), \widehat{B}(t, (i-1)T, (V^{\lambda}_{(i-1)T}, X^{\lambda}_{(i-1)T}))) \nonumber\\
	&=&  \sum_{i = 1}^n \mathfrak{W}_p( \widehat{B}(t, iT, (V^{\lambda}_{iT}, X^{\lambda}_{iT})) , \widehat{B}(t, iT, \widehat{B}(iT, (i-1)T, (V^{\lambda}_{(i-1)T}, X^{\lambda}_{(i-1)T}))) )\nonumber\\
	&\le&  C_* \sum_{i = 1}^n e^{-c_* \lambda (t-iT)} \mathcal{W}^{1/p}_{\rho}( \mathcal{L} (V^{\lambda}_{iT}, X^{\lambda}_{iT}) , \mathcal{L} (\widehat{B}(iT, (i-1)T, (V^{\lambda}_{(i-1)T}, X^{\lambda}_{(i-1)T})) )) \nonumber \\\label{long}
	\end{eqnarray}
	noting the rate of contraction of $(\widehat{V}_t, \widehat{X}_t)$ is $e^{-c_* \lambda t}$. Using Lemma \ref{lemma:rho_to_w}, we obtain
	\begin{eqnarray*}
		&&\mathcal{W}_{\rho}(( \mathcal{L} (V^{\lambda}_{iT}, X^{\lambda}_{iT}) , \mathcal{L}(\widehat{V}(iT, (i-1)T, (V^{\lambda}_{(i-1)T},X^{\lambda}_{(i-1)T})), \widehat{X}(iT, (i-1)T, (V^{\lambda}_{(i-1)T},X^{\lambda}_{(i-1)T}))))) \\
		&\le& c_{17} \left( 1 + \varepsilon_c \sqrt{E\mathcal{V}^2(V^{\lambda}_{iT}, X^{\lambda}_{iT})}    + \varepsilon_c \sqrt{E\mathcal{V}^2(\widehat{V}(iT, (i-1)T, (V^{\lambda}_{(i-1)T},X^{\lambda}_{(i-1)T})), \widehat{X}(iT, (i-1)T, (V^{\lambda}_{(i-1)T},X^{\lambda}_{(i-1)T})) )} \right) \\
		&\times&  \mathfrak{W}_2((V^{\lambda}_{iT}, X^{\lambda}_{iT}),(\widehat{V}(iT, (i-1)T, (V^{\lambda}_{(i-1)T},X^{\lambda}_{(i-1)T})), \widehat{X}(iT, (i-1)T, (V^{\lambda}_{(i-1)T},X^{\lambda}_{(i-1)T}))))\\
		&\le& c_{18} \left( E^{1/2} \|V^{\lambda}_{iT} - \widehat{V}(iT, (i-1)T, (V^{\lambda}_{(i-1)T},X^{\lambda}_{(i-1)T}))\|^2 + E^{1/2}\left[\| X^{\lambda}_{iT} - \widehat{X}(iT, (i-1)T, (V^{\lambda}_{(i-1)T},X^{\lambda}_{(i-1)T})) \|^2  \right]  \right),
	\end{eqnarray*}
	where
	\begin{eqnarray*}
		c_{18} &=& c_{17} \left( 1 + \varepsilon_c \sup_{k \in \mathbb{N}} \sqrt{E\mathcal{V}^2(V^{\lambda}_{k}, X^{\lambda}_{k})}  \right.  \\
		&&\left.  + \varepsilon_c \sup_{k \in \mathbb{N}} \sqrt{E\mathcal{V}^2(\widehat{V}(kT, (k-1)T, (V^{\lambda}_{(k-1)T},X^{\lambda}_{(k-1)T})), \widehat{X}(kT, (k-1)T, (V^{\lambda}_{(k-1)T},X^{\lambda}_{(k-1)T})))}  \right).
	\end{eqnarray*}
	Now, we compute 
	\begin{eqnarray}
	&&\|V^{\lambda}_{iT} - \widehat{V}(iT, (i-1)T, (V^{\lambda}_{(i-1)T},X^{\lambda}_{(i-1)T}))\| \nonumber\\
	&& \le  \|  V^{\lambda}_{iT - 1} - \widehat{V}(iT-1, (i-1)T, (V^{\lambda}_{(i-1)T},X^{\lambda}_{(i-1)T})) \| \nonumber \\
	&& + \lambda \gamma  \left\| V^{\lambda}_{iT - 1} - \widehat{V}(iT-1,(i-1)T,(V^{\lambda}_{(i-1)T},X^{\lambda}_{(i-1)T})) \right\| \nonumber \\
	&& + \lambda \gamma \left\| \int_{iT-1}^{iT}{\left( \widehat{V}(iT-1,(i-1)T,(V^{\lambda}_{(i-1)T},X^{\lambda}_{(i-1)T})) - \widehat{V}(t,(i-1)T,(V^{\lambda}_{(i-1)T},X^{\lambda}_{(i-1)T}))\right) dt}  \right\| \nonumber \\
	&& + \lambda \left\| g(X^{\lambda}_{iT-1}, U_{\mathbf{z},iT-1}) - \int_{(iT-1)}^{iT}{\nabla F_{\mathbf{z}}(\widehat{X}(t,(i-1)T,(V^{\lambda}_{(i-1)T},X^{\lambda}_{(i-1)T})))dt} \right\| \nonumber\\
	&& \qquad \qquad \qquad + \sqrt{\lambda} \|\xi_{iT} - (B^{\lambda}_{iT} - B^{\lambda}_{iT-1} )\| \label{eq_very_long} .
	\end{eqnarray}
	In $L^2$ norm, the first and second terms of (\ref{eq_very_long}) is bounded by $(c_2 + c_7) \sqrt{\lambda} + c_3 \sqrt{\delta}$, see (\ref{algo_hat}) and the fifth term is estimated by $\sqrt{\lambda}$. 
	We consider the third term in (\ref{eq_very_long}). From the dynamics of $\widehat{V}$, we find that for $iT-1 \le t \le iT$,
	\begin{eqnarray*}
		&&\widehat{V}(iT-1,(i-1)T,(V^{\lambda}_{(i-1)T},X^{\lambda}_{(i-1)T})) - \widehat{V}(t,(i-1)T,(V^{\lambda}_{(i-1)T},X^{\lambda}_{(i-1)T})) \\
		&& = \lambda \int_{iT-1}^t \left(  \gamma \widehat{V}(s,(i-1)T,(V^{\lambda}_{(i-1)T},X^{\lambda}_{(i-1)T})) + \nabla F_{\mathbf{z}}(\widehat{X}(s,(i-1)T,(V^{\lambda}_{(i-1)T},X^{\lambda}_{(i-1)T})))\right) ds \\
		&& - \sqrt{2 \gamma \lambda \beta^{-1}}\left( B^{\lambda}_t -  B^{\lambda}_{iT-1}  \right). 
	\end{eqnarray*}
	H\"older's inequality yields
	\begin{eqnarray*}
		&&E\left[ \|\widehat{V}(iT-1,(i-1)T,(V^{\lambda}_{(i-1)T},X^{\lambda}_{(i-1)T})) - \widehat{V}(t,(i-1)T,(V^{\lambda}_{(i-1)T},X^{\lambda}_{(i-1)T})) \|^2 \right] \\
		&&  \le 3\lambda^2 \gamma^2 \int_{iT-1}^t{ E \left[ \|\widehat{V}(s,(i-1)T,(V^{\lambda}_{(i-1)T},X^{\lambda}_{(i-1)T})) \|^2\right]  ds }\\
		&&  + 3\lambda^2 \int_{iT-1}^t{ E\left[ \left\| \nabla F_{\mathbf{z}}(\widehat{X}(s,(i-1)T,(V^{\lambda}_{(i-1)T},X^{\lambda}_{(i-1)T}))) \right\|^2\right] ds} + 6 \gamma \beta^{-1} \lambda \\
		&& \le c_{14} \lambda,
	\end{eqnarray*}
	where the last inequality uses Lemma \ref{lem_uniform2} and Assumption \ref{as_lip} and $c_{14}:= 3 \gamma^2 C^c_v + 6M^2 C^c_x + 6B^2  + 6 \gamma \beta^{-1}$. For the fourth term of (\ref{eq_very_long}), we have
	\begin{eqnarray*}
		&& E \left[ \left\| g(X^{\lambda}_{iT-1}, U_{\mathbf{z},iT-1}) - \int_{(iT-1)}^{iT}{\nabla F_{\mathbf{z}}(\widehat{X}(t,(i-1)T,(V^{\lambda}_{(i-1)T},X^{\lambda}_{(i-1)T})))dt} \right\|^2\right]  \\
		&& \le 2 E \left[ \| g(X^{\lambda}_{iT-1}, U_{\mathbf{z},iT-1}) - \nabla F_{\mathbf{z}}(X^{\lambda}_{iT-1})  \|^2\right] \\
		&& + 2E \left[ \left\|  \int_{(iT-1)}^{iT}{\nabla F_{\mathbf{z}}(X^{\lambda}_{iT-1}) - \nabla F_{\mathbf{z}}(\widehat{X}(t,(i-1)T,(V^{\lambda}_{(i-1)T},X^{\lambda}_{(i-1)T})))dt}\right\|^2\right]  \\
		&& \le 2E\left[  \| g(X^{\lambda}_{iT-1}, U_{\mathbf{z},iT-1}) - \nabla F_{\mathbf{z}}(X^{\lambda}_{iT-1})  \|^2\right] \\
		&& +  2M^2E \left[  \int_{(iT-1)}^{iT}{\left\|X^{\lambda}_{iT-1} - \widehat{X}(t,(i-1)T,(V^{\lambda}_{(i-1)T},X^{\lambda}_{(i-1)T}))\right\|^2  dt}\right] \\
		&& \le 2 \delta (M^2 C^a_x + B^2) + 2M^2 (2(c_2 + c_7)^2 \lambda + 2c_3^2 \delta)\\
		&& \le c_{15}(\lambda + \delta),
	\end{eqnarray*}
	where the last inequality uses Assumption \ref{as_variance}, Lemma \ref{lem_uniform2}, and (\ref{algo_hat}) and $c_{15}:= \max\{ 2 (M^2 C^a_x + B^2) + 4M^2c_3^2, 4M^2(c_2 + c_7)^2 \}$. A similar estimate holds for 
	$$E^{1/2}\left[\| X^{\lambda}_{iT} - \widehat{X}(iT, (i-1)T, (V^{\lambda}_{(i-1)T},X^{\lambda}_{(i-1)T})) \|^2  \right].$$
	Letting $c_{16}:= \max\{ (c_2 + c_7), c_3, \sqrt{c_{14}}, \sqrt{c_{15}}\}$, the estimation (\ref{long}) continues as
	\begin{eqnarray}
	\mathfrak{W}_p( \widehat{A}_t, \widehat{B}(t,0,(v_0,x_0))) &\le&  \sum_{i = 1}^n C_*e^{-c_*(n-i)} \left( c_{18} c_{16} (\sqrt{\lambda} + \sqrt{\delta}) \right)^{1/p} \nonumber\\
	&\le&   C_* \left( c_{18} c_{16} \right)^{1/p} \frac{e^{-c_*}}{1-e^{-c_*}} (\lambda^{1/(2p)} + \delta^{1/(2p)}). \label{hat_hat}
	\end{eqnarray}
	Therefore, from (\ref{algo_tilde}), (\ref{tilde_hat}), (\ref{hat_hat}), the triangle inequality implies for $nT \le k < (n+1)T$,
	\begin{eqnarray*}
		&&\mathfrak{W}_p((V^{\lambda}_k,X^{\lambda}_k) , (\widehat{V}(k,0,(v_0,x_0)), \widehat{X}(k,0,(v_0,x_0)))) \\
		&\le&  \mathfrak{W}_p((V^{\lambda}_k, X^{\lambda}_k), (\tilde{V}^{\lambda}_k, \tilde{X}^{\lambda}_k )) + \mathfrak{W}_p((\tilde{V}^{int}_{k}, \tilde{X}^{int}_{k}) , (\widehat{V}_{k}, \widehat{X}_{k}) )\\
		&+& \mathfrak{W}_p( (\widehat{V}_{k}, \widehat{X}_{k}) , (\widehat{V}(k,0,(v_0,x_0)), \widehat{X}(k,0,(v_0,x_0))) ) \\
		&\le& \tilde{C} (\lambda^{1/(2p)} + \delta^{1/(2p)}),
	\end{eqnarray*}
	where $\tilde{C} = 2\max\{c_2, c_3, c_7 ,C_* \left( c_{18} c_{16} \right)^{1/p} \frac{e^{-c_*}}{1-e^{-c_*}}\}$. The proof is complete.
\end{proof}
\begin{remark}\label{re_sampling_data}
	It is important to remark from the proof above that the data structure of $\mathbf{Z}$ can be \emph{arbitrary}, and only the independence of random elements $U_{\mathbf{z},k}, k \in \mathbb{N}$ is used. 
\end{remark}
\begin{lemma}\label{lem_Xi}
	The quantity $\Xi_n$ defined in (\ref{eq_defi_Xi}) has second moments and $$\sup_{n\in \mathbb{N}} E[\Xi_n^2] < \infty.$$
\end{lemma}
\begin{proof}
	Noting that for each $nT \le i < (n+1)T -1$, the random variable $\tilde{X}^{\lambda}_i$ is $\mathcal{H}_{nT}$-measurable. Using Assumption \ref{as_variance}, the Cauchy–Schwarz inequality implies
	\begin{eqnarray*}
		E[\Xi_n^2] &\le& T 	\sum_{i = nT}^{(n+1)T - 1} E \left[ \left\| g_{i,nT}(\tilde{X}^{\lambda}_i) - \nabla F_{\mathbf{z}}(\tilde{X}^{\lambda}_i) \right\|^2 \right]  \\
		&=& T \sum_{i = nT}^{(n+1)T - 1} E\left[  \left\| E\left[ g(\tilde{X}^{\lambda}_i,U_{\mathbf{z},k})| \mathcal{H}_{nT} \right] - \nabla F_{\mathbf{z}}(\tilde{X}^{\lambda}_i) \right\|^2 \right] \\
		&\le& T \sum_{i = nT}^{(n+1)T - 1} E\left[ E\left[  \left\|  g(\tilde{X}^{\lambda}_i,U_{\mathbf{z},k}) - \nabla F_{\mathbf{z}}(\tilde{X}^{\lambda}_i)\right\|^2 | \mathcal{H}_{nT} \right] \right] \\
		&\le& 2T\delta \sum_{i = nT}^{(n+1)T - 1}  (M^2 E\left[ \|\tilde{X}^{\lambda}_i\|^2\right]  + B^2)\\
		&\le& 2T^2\delta(M^2 C^a_x + B^2),
	\end{eqnarray*}
	where the last inequality uses Lemma \ref{lem_uniform2}.	
\end{proof}
This lemma provides variance control for the algorithm. Each term in $\Xi_n$ has an error of order $\delta$, the total variance in $\Xi_n$ is of order $T \delta$. However, unlike \cite{raginsky}, \cite{gao}, our technique does not accumulate variance errors over time, as shown in (\ref{algo_tilde}). Recently in \cite{six}, the authors imposed no condition for variance of the estimated gradient, but employ the conditional $L$-mixing property of data stream, and hence variance is controlled by the decay of mixing property, see their Lemma 8.6. 
\begin{lemma}\label{lem_tilde_hat}
	For every $nT \le t < (n+1)T$, it holds that
	$$E^{1/2}\left[ \|\tilde{V}^{int}_t - \widehat{V}_t\|^2\right]  \le c_7 \sqrt{\lambda}, \qquad E^{1/2}\left[ \| \tilde{X}^{int}_t - \widehat{X}_t \|^2\right] \le c_7 \sqrt{\lambda}.$$
\end{lemma}
\begin{proof}
	Noting that $\tilde{V}^{int}_{nT} = \widehat{V}_{nT} = V^{\lambda}_{nT}$, we use the triangle inequality and Assumption \ref{as_lip} to estimate
	\begin{eqnarray}
	\|\tilde{V}^{int}_t - \widehat{V}_t\| &\le& \lambda \gamma    \int_{nT}^t{  \left\| \tilde{V}^{int}_{\lfloor s \rfloor} - \widehat{V}_s\right\| ds}   + \lambda  \int_{nT}^t{ \left\|  \nabla F_{\mathbf{z}}(\tilde{X}^{int}_{\lfloor s \rfloor}) -\nabla F_{\mathbf{z}}(\widehat{X}_s) \right\| ds } \nonumber \\
	&\le& \lambda \gamma    \int_{nT}^t{  \left\|\tilde{V}^{int}_s - \widehat{V}_s\right\|  ds}  + \lambda M  \int_{nT}^t{ \left\|  \tilde{X}^{int}_s -\widehat{X}_s \right\| ds } \nonumber\\
	&&  + \lambda \gamma    \int_{nT}^t{  \left\| \tilde{V}^{int}_{\lfloor s \rfloor} - \tilde{V}^{int}_s\right\| ds}   + \lambda  M  \int_{nT}^t{ \left\|   \tilde{X}^{int}_{\lfloor s \rfloor} -\tilde{X}^{int}_s \right\| ds}. \label{eq_1}
	\end{eqnarray}
	For notational convenience, we define for every $nT \le t < (n+1)T$
	\begin{eqnarray*}
		I_t := \|\tilde{V}^{int}_t - \widehat{V}_t\|, \qquad J_t:= \left\|  \tilde{X}^{int}_t -\widehat{X}_t \right\|.
	\end{eqnarray*}
	Then (\ref{eq_1}) becomes
	\begin{equation}\label{eq_2}
	I_t \le \lambda \gamma \int_{nT}^t{I_sds} + \lambda M \int_{nT}^t{J_sds} + \lambda \gamma    \int_{nT}^t{  \left\| \tilde{V}^{int}_{\lfloor s \rfloor} - \tilde{V}^{int}_s\right\| ds}   + \lambda  M  \int_{nT}^t{ \left\|   \tilde{X}^{int}_{\lfloor s \rfloor} -\tilde{X}^{int}_s \right\| ds}.
	\end{equation}
	Furthermore,
	\begin{eqnarray}
	J_t &\le& \lambda \int_{nT}^t{\|\tilde{V}^{int}_s - \widehat{V}_s\|  ds} + \lambda \int_{nT}^t{\|\tilde{V}^{int}_{\lfloor s \rfloor } -\tilde{V}^{int}_s\|  ds} \nonumber\\
	&\le& \lambda \int_{nT}^t{ I_sds } + \lambda \int_{nT}^t{\|\tilde{V}^{int}_{\lfloor s \rfloor } -\tilde{V}^{int}_s\|  ds}.\label{eq_3}
	\end{eqnarray}
	We estimate
	\begin{eqnarray*}
		\left\| \tilde{V}^{int}_{\lfloor t \rfloor} - \tilde{V}^{int}_t\right\| &\le&  \lambda \gamma \int_{\lfloor t \rfloor}^t {\|\tilde{V}^{int}_{\lfloor s \rfloor}\| ds} + \lambda \int_{\lfloor t \rfloor}^t{\|\nabla F_{\mathbf{z}}(\tilde{X}^{int}_{\lfloor s \rfloor})\|ds} + \sqrt{2\gamma \lambda \beta^{-1}} \|B^{\lambda}_t - B^{\lambda}_{\lfloor t \rfloor} \|.
	\end{eqnarray*}
	Noting that $0 \le t - \lfloor t \rfloor \le 1 $, the Cauchy-Schwarz inequality and Lemma \ref{lem_quad} imply
	\begin{eqnarray*}
		\left\| \tilde{V}^{int}_{\lfloor t \rfloor} - \tilde{V}^{int}_t\right\|^2 &\le&  3\lambda^2 \gamma^2 \int_{\lfloor t \rfloor}^t {\|\tilde{V}^{int}_{\lfloor s \rfloor}\|^2 ds} + 6\lambda^2 M^2 \int_{\lfloor t \rfloor}^t{\|\tilde{X}^{int}_{\lfloor s \rfloor}\|^2ds} \\
		&+& 6 \lambda^2 B^2 + 6\gamma \lambda \beta^{-1} \|B^{\lambda}_t - B^{\lambda}_{\lfloor t \rfloor} \|^2.
	\end{eqnarray*}
	Taking expectation both sides and noting that $(\tilde{V}^{int}_{k}, \tilde{X}^{int}_{k})$ has the same distribution as $(\tilde{V}^{\lambda}_{k}, \tilde{X}^{\lambda}_{k}), k \in \mathbb{N}$, Lemma \ref{lem_uniform2} leads to
	\begin{eqnarray}
	E\left[ \left\| \tilde{V}^{int}_{\lfloor t \rfloor} - \tilde{V}^{int}_t\right\|^2\right]  &\le& 3\lambda^2 \gamma^2 C^a_v + 6\lambda^2 M^2 C^a_x + 6 \lambda^2B^2 + 6\gamma \beta^{-1} \lambda \nonumber \\
	&\le& c_8 \lambda, \label{eq_4}
	\end{eqnarray}
	for $c_8: = 3 \gamma^2 C^a_v + 6 M^2 C^a_x + 6 B^2 + 6\gamma \beta^{-1} $. Similarly,
	\begin{eqnarray}
	E\left[ \left\| \tilde{X}^{int}_{\lfloor t \rfloor} -\tilde{X}^{int}_t \right\|^2\right]  = \lambda^2 \int_{\lfloor t \rfloor}^t{E\left[ \|\tilde{V}^{int}_{\lfloor s \rfloor }\|^2\right]  ds} \le \lambda^2 C^a_v. \label{eq_5}
	\end{eqnarray}
	Taking squares and expectation of (\ref{eq_2}), (\ref{eq_3}), applying (\ref{eq_4}), (\ref{eq_5}) we obtain for $nT \le t < (n+1)T $
	\begin{eqnarray*}
		E\left[ I^2_t\right]  &\le& 4\lambda \gamma^2 \int_{nT}^t{E\left[ I^2_s\right] ds} + 4\lambda M^2 \int_{nT}^t{E\left[ J^2_s\right] ds} + c_9 \lambda,\\
		E\left[ J^2_t\right]  &\le& 2\lambda \int_{nT}^t{ E\left[ I^2_s\right] ds } + c_9\lambda,
	\end{eqnarray*}
	where $c_9 := \max\{4 \gamma^2 c_8 + 4M^2 C^a_v, 2c_8\}$.
	Summing up two inequalities yields
	$$E[I^2_t + J^2_t] \le c_{10} \lambda  \int_{nT}^t{E[I^2_s + J^2_s]ds} + 2c_9 \lambda  $$
	where $c_{10}:= \max\{ 4 \gamma^2 + 2, 4M^2 \}$
	and then Gronwall's lemma shows
	$$E[I^2_t + J^2_t] \le2c_9 \lambda e^{c_{10}}.$$
	noting that $t \mapsto E[I^2_t + J^2_t]$ is continuous. The proof is complete by setting $c_7 = \sqrt{2c_9 e^{c_{10}}}$, which is of order $\sqrt{d}$.
\end{proof}

\subsection{Proof of Theorem \ref{thm_main}}\label{proof_thm_main}
Denote $\mu_{\mathbf{z},k}:= \mathcal{L}((V^{\lambda}_k,X^{\lambda}_k)|\mathbf{Z} = \mathbf{z})$. Let $(\widehat{X},\widehat{V})$ and $(\widehat{X}^*,\widehat{V}^*)$ be such that $\mathcal{L}((\widehat{X},\widehat{V})|\mathbf{Z} = \mathbf{z}) = \mu_{\mathbf{z},k}$ and $\mathcal{L}(\widehat{X}^*_{\mathbf{z}},\widehat{V}^*_{\mathbf{z}}) = \pi_{\mathbf{z}}$. We decompose the population risk by
\begin{eqnarray}
E\left[ F(\widehat{X})\right] - F^* &=& \left( E\left[ F(\widehat{X})\right]  - E\left[ F(\widehat{X}^*_{\mathbf{z}})\right]  \right) + \left( E\left[ F(\widehat{X}^*_{\mathbf{z}})\right]  - E\left[ F_{\mathbf{Z}}(\widehat{X}^*_{\mathbf{Z}})\right]  \right) \nonumber \\
&+& \left(E\left[ F_{\mathbf{Z}}(\widehat{X}^*_{\mathbf{Z}})\right]  - F^* \right). \label{eq_decom_risk}
\end{eqnarray}
\subsubsection{The first term $\mathcal{T}_1$}
The first term in the right hand side of (\ref{eq_decom_risk}) is rewritten as
$$E\left[ F(\widehat{X})\right]  - E\left[ F(\widehat{X}^*)\right]  = \int_{\mathcal{Z}^n} \mu^{\otimes n}(d\mathbf{z}) \left( \int_{\mathbb{R}^{2d}} F_{\mathbf{z} }(x)\mu_{\mathbf{z},k}(dx,dv)  - \int_{\mathbb{R}^{2d}} F_{\mathbf{z} }(x)\pi_{\mathbf{z}}(dx,dv) \right),$$
where $\mu^{\otimes n}$ is the product of laws of independent random variables $Z_1, ...,Z_n$. By Assumptions \ref{as_f_bound} and \ref{as_lip}, the function $F_{\mathbf{z}}$ satisfies
$\|\nabla F_{\mathbf{z}}(x)\| \le M\|x\| + B.$ Using Lemma \ref{lem_linear_bound}, we have
$$\left| \int_{\mathbb{R}^{2d}} F_{\mathbf{z} }(x)\mu_{\mathbf{z},k}(dx,dv)  - \int_{\mathbb{R}^{2d}} F_{\mathbf{z} }(x)\pi_{\mathbf{z}}(dx,dv) \right| \le (M \sigma + B) \mathcal{W}_p(\mu_{\mathbf{z},k}, \pi_{\mathbf{z}}),$$
where $p>1, q \in \mathbb{N}, 1/p + 1/(2q) = 1,$
$$\sigma = \max \left\lbrace  \left( \int_{\mathbb{R}^{2d}}{\|x\|^{2q} \mu_{\mathbf{z},k}(dx,dv)}\right) ^{1/(2q)}, \left( \int_{\mathbb{R}^{2d}}{\|x\|^{2q} \pi_{\mathbf{z}}(dx,dv)} \right) ^{1/(2q)} \right\rbrace < \infty $$
by Lemma \ref{lemma:moment_bound}. On the other hand, Theorems \ref{thm_algo_aver} and \ref{thm_contraction} imply
\begin{eqnarray}
&& \mathcal{W}_p(\mu_{\mathbf{z},k}, \pi_{\mathbf{z}}) \nonumber\\ 
&& \le \mathcal{W}_p(\mathcal{L}((V^{\lambda}_k,X^{\lambda}_k)|\mathbf{Z} = \mathbf{z}), \mathcal{L}((\widehat{V}(k,0,v_0), \widehat{X}(k,0,x_0))|\mathbf{Z} = \mathbf{z})) \nonumber \\
&& + \mathcal{W}_p(\mathcal{L}((\widehat{V}(k,0,v_0), \widehat{X}(k,0,x_0))|\mathbf{Z} = \mathbf{z}), \pi_{\mathbf{z}} )  \nonumber \\
&& \le \tilde{C} (\lambda^{1/(2p)} + \delta^{1/(2p)}) +  C_*  \left( \mathcal{W}_{\rho}(\mu_0, \pi_{\mathbf{z}})\right)^{1/p}  \exp(- c_* k \lambda )\label{sampling_err}.
\end{eqnarray}
Therefore, an upper bound for $\mathcal{T}_1$ is given by
\begin{equation*}
\mathcal{T}_1 \le (M \sigma + B) \left( \tilde{C} (\lambda^{1/(2p)} + \delta^{1/(2p)}) +  C_*  \left( \mathcal{W}_{\rho}(\mu_0, \pi_{\mathbf{z}})\right)^{1/p} \exp(- c_* k \lambda )  \right).
\end{equation*}
\subsubsection{The second term $\mathcal{T}_2$}
Since the $x$-marginal of $\pi_{\mathbf{z}}(dx,dv)$ is $\pi_{\mathbf{z}}(dx)$, the Gibbs measure of (\ref{langevin}), we compute
$$\int_{\mathbb{R}^{2d}}{F_{\mathbf{z}}(x) \pi_{\mathbf{z}}(dx,dv)} = \int_{\mathbb{R}^d}{F_{\mathbf{z}}(x)\pi_{\mathbf{z}}(dx)}.$$
Therefore the argument in \cite{raginsky} is adopted,
$$E\left[ F(\widehat{X}^*)\right]  - E\left[ F_{\mathbf{Z}}(\widehat{X}^*)\right]  \le \frac{4\beta c_{LS}}{n}\left( \frac{M^2}{m}(b+d/\beta) + B^2 \right).$$
The constant $c_{LS}$ comes from the logarithmic Sobolev inequality for $\pi_{\mathbf{z}}$ and
$$ c_{LS} \le  \frac{2m^2 + 8M^2}{m^2M\beta} + \frac{1}{\lambda_*} \left( \frac{6M(d+\beta)}{m} + 2 \right),$$
where $\lambda_*$ is the uniform spectral gap for the overdamped Langevin dynamics
$$\lambda_* = \inf_{\mathbf{z} \in \mathcal{Z}^n} \inf \left\lbrace \frac{\int_{\mathbb{R}^d}\|\nabla g\|^2d\pi_{\mathbf{z}}}{\int_{\mathbb{R}^d}g^2d\pi_{\mathbf{z}}}: g \in C^1(\mathbb{R}^d) \cap L^2(\pi_{\mathbf{z}}), g \ne 0, \int_{\mathbb{R}^d}{gd\pi_{\mathbf{z}}} = 0 \right\rbrace.$$
\begin{remark}\label{re_extend_T2}
	One can also find an upper bound for $\mathcal{T}_2$ when the data $\mathbf{z}$ is a realization of some non-Makovian processes. For example, if we assume that $f$ is Lipschitz on the second variable $z$ and $\mathbf{Z}$ satisfies a certain mixing property discussed in \cite{chau2016fixed}) then the term $\mathcal{T}_2$ is bounded by $1/\sqrt{n}$ times a constant, see Theorem 2.5 therein.
\end{remark}
\subsubsection{The third term $\mathcal{T}_3$}
For the third term, we follow \cite{raginsky}. Let $x^*$ be any minimizer of $F(x)$. We compute
\begin{eqnarray}
E\left[ F_{\mathbf{Z}}(\widehat{X}^*)\right]  - F^* &=& E\left[ F_{\mathbf{Z}}(\widehat{X}^*) - \min_{x \in \mathbb{R}^d} F_{\mathbf{Z}}(x) \right] + E\left[ \min_{x \in \mathbb{R}^d} F_{\mathbf{Z}}(x)  - F_{\mathbf{Z}}(x^*) \right] \nonumber\\
&\le& E\left[ F_{\mathbf{Z}}(\widehat{X}^*) - \min_{x \in \mathbb{R}^d} F_{\mathbf{Z}}(x) \right] \nonumber\\
&\le& \frac{d}{2\beta} \log\left( \frac{eM}{m}\left( \frac{b\beta}{d} + 1 \right)  \right), 
\end{eqnarray}
where the last inequality comes from Proposition 3.4 of \cite{raginsky}. The condition $\beta \ge 2m$ is not used here, see the explanation in Lemma 16 of \cite{gao}.
\section{Technical lemmas}
\begin{lemma}\label{lem_quad}
	Under Assumptions \ref{as_f_bound}, \ref{as_lip}, for any $x \in \mathbb{R}^d$ and $z \in \mathcal{U}$, 
	$$ \|\nabla f(x,z)\| \le M \|x \| + B,$$
	and $$ \frac{m}{3}\|x\|^2 - \frac{b}{2} \log 3 \le f(x,z) \le \frac{M}{2} \|x \|^2  + B\|x\| + A_0. $$
\end{lemma}
\begin{proof}
	See Lemma 2 of \cite{raginsky}.
\end{proof}
The next lemma generalizes continuity for functions of quadratic growth in Wasserstein distances given in \cite{polyanskiy2016wasserstein}. 
\begin{lemma}\label{lem_linear_bound}Let $\mu, \nu$ be two probability measures on $\mathbb{R}^{2d}$ with finite second moments and let $G: \mathbb{R}^{2d} \to \mathbb{R}$ be  a $C^1$ function with
	$$\| \nabla G(w) \| \le c_1\|w\| + c_2 $$
	for some $c_1 > 0, c_2 \ge 0 $. Then for $p > 1, q>1 $ such that $1/p + 1/q = 1$, we have
	$$\left| \int_{\mathbb{R}^{2d}}{G d\mu} - \int_{\mathbb{R}^{2d}}{G d\nu} \right| \le (c_1 \sigma + c_2) \mathcal{W}_p(\mu,\nu),$$
	where 
	$$\sigma = \frac{1}{2} \max\left\lbrace  \left( \int_{\mathbf{R}^{2d}}{\|v\|^q\nu(dv)}  \right)^{1/q} , \left( \int_{\mathbf{R}^{2d}}{\|u\|^q\mu(du)}  \right)^{1/q}  \right\rbrace .$$
\end{lemma}
\begin{proof}
	Using the Cauchy-Schwartz inequality, we compute
	\begin{eqnarray*}
		|G(u) - G(v)| &=& \left| \int_{0}^1{\left\langle \nabla G(tv + (1-t)u), u-v \right\rangle dt} \right|  \\
		&\le& \left| \int_{0}^1{ (c_1 t\|v\| + c_1(1-t)\|u\| + c_2)\| u-v\| dt } \right| \\
		&=& (c_1\|v\|/2 + c_1\|u\|/2 +c_2) \| u-v\|.
	\end{eqnarray*}	
	Then for any $\xi \in \Pi(\mu,\nu)$ we have
	\begin{eqnarray*}
		\left| \int_{\mathbf{R}^{2d}}{G(u)\mu(du)} - \int_{\mathbf{R}^{2d}}{G(v)\nu(dv)}\right| &\le& \int_{\mathbf{R}^{2d}}{ (c_1\|v\|/2 + c_1\|u\|/2 +c_2) \| u-v\| \xi(du,dv)  }\\
		&\le& \frac{c_1}{2}\left( \int_{\mathbf{R}^{2d}}{\|v\|^q\nu(dv)}  \right)^{1/q} \left( \int_{\mathbf{R}^{2d}}{\|u-v\|^p} \xi(du,dv) \right)^{1/p} \\
		&+& \frac{c_1}{2}\left( \int_{\mathbf{R}^{2d}}{\|u\|^q\mu(du)}  \right)^{1/q} \left( \int_{\mathbf{R}^{2d}}{\|u-v\|^p} \xi(du,dv) \right)^{1/p}\\
		&+& c_2\left( \int_{\mathbf{R}^{2d}}{\|u-v\|^p} \xi(du,dv) \right)^{1/p}.
	\end{eqnarray*}
	Since this inequality holds true for any $\xi \in \Pi(\mu,\nu)$, the proof is complete. 
\end{proof}
\begin{lemma}\label{lem_uniform2}
	The continuous time processes (\ref{eq_V}),(\ref{eq_X}) are uniformly bounded in $L^2$, more precisely,
	\begin{eqnarray*}
		\sup_{t \ge 0} E_{\mathbf{z}}\left[ \|X_t\|^2\right]  &\le& C^c_x: = \frac{8 }{(1- 2 \lambda_c) \beta \gamma^2} \left( \int_{\mathbb{R}^{2d}}{\mathcal{V}(x,v)d\mu_0(x,v)} + \frac{5(d + A_c)}{\lambda_c} \right)  < \infty, \\
		\sup_{t \ge 0} E_{\mathbf{z}}\left[ \|V_t\|^2\right]  &\le& C^c_v:= \frac{4}{(1- 2 \lambda_c) \beta } \left( \int_{\mathbb{R}^{2d}}{\mathcal{V}(x,v)d\mu_0(x,v)} + \frac{5(d + A_c)}{\lambda_c} \right) < \infty.
	\end{eqnarray*}
	For $0 < \lambda \le \min\left\lbrace \frac{\gamma}{K_2}\left( \frac{d + A_c}{\beta}, \frac{\gamma \lambda_c}{2K_1} \right)  \right\rbrace $, where
	\begin{equation*}
	K_1 := \max\left\lbrace \frac{32M^2(\frac{1}{2}+\gamma + \delta)}{(1-2 \lambda_c) \beta \gamma^2}, \frac{8(\frac{M}{2} + \frac{\gamma^2}{4} - \frac{\gamma^2 \lambda_c}{4} + \gamma)}{\beta(1-2\lambda_c)} \right\rbrace 
	\end{equation*} 
	and
	\begin{equation*}
	k_2:= 2B^2\left( \frac{1}{2} + \gamma + \delta \right), 
	\end{equation*}the SGHMC (\ref{eq_V_dis_appr}),(\ref{eq_X_dis_appr}) satisfy
	\begin{eqnarray*}
		\sup_{k \in \mathbb{N} } E_{\mathbf{z}}\left[ \|X^{\lambda}_k\|^2\right]  &\le& C^a_x:= \frac{8 }{(1- 2 \lambda_c) \beta \gamma^2} \left( \int_{\mathbb{R}^{2d}}{\mathcal{V}(x,v)d\mu_0(x,v)} + \frac{8(d + A_c)}{\lambda_c} \right)   < \infty, \\
		\sup_{k \in \mathbb{N} } E_{\mathbf{z}}\left[ \|V^{\lambda}_k\|^2\right]  &\le&   C^a_v:= \frac{4}{(1- 2 \lambda_c) \beta } \left( \int_{\mathbb{R}^{2d}}{\mathcal{V}(x,v)d\mu_0(x,v)} + \frac{8(d + A_c)}{\lambda_c} \right) < \infty. 
	\end{eqnarray*}
	Furthermore, the processes defined in (\ref{pro_tilde}), (\ref{proc_hat}) are also uniformly bounded in $L^2$ with the upper bounds $C^c_v, C^c_x, C^a_v, C^a_x$, respectively.
\end{lemma}
\begin{proof}
	The uniform boundedness in $L^2$ of the processes in (\ref{eq_V}), (\ref{eq_X}), (\ref{eq_V_dis_appr}), (\ref{eq_X_dis_appr}) are given in Lemma 8 of \cite{gao}. From (A.4) of \cite{gao}, it holds that
	\begin{equation}\label{ineq_lower_V}
	\mathcal{V}(v,x) \ge \max\left\lbrace  \frac{1}{8}(1-2\lambda_c)\beta \gamma^2\|x\|^2, \frac{\beta}{4}(1-2\lambda_c)\|v\|^2 \right\rbrace .
	\end{equation} 
	Using the notations in their Lemma 8, we denote
	$$L_t = E_{\mathbf{z}}\left[ \mathcal{V}(V_t,X_t) \right], \qquad L_2(k) = E_{\mathbf{z}}\left[ \mathcal{V}(V^{\lambda}_k,X^{\lambda}_k)/\beta \right],$$
	then the following relations hold
	\begin{eqnarray}
	L_t &\le& L_s e^{-\gamma \lambda_c (t-s)} + \frac{d + A_c}{\lambda_c}(1- e^{\gamma \lambda_c(t-s)}), \qquad \text{ for } s \le t, \label{L_c}\\
	L_2(k) &\le& L_2(j) + \frac{4(d/\beta + A_c/\beta)}{\lambda_c} \qquad \text{ for } j\le k.\label{L_d}
	\end{eqnarray}
	Taking $j=0$ in (\ref{L_d}) gives
	\begin{equation}\label{ineq_bound_V_algo}
	E_{\mathbf{z}}\left[ \mathcal{V}(V^{\lambda}_k,X^{\lambda}_k) \right] \le E_{\mathbf{z}}\left[ \mathcal{V}(V^{\lambda}_0,X^{\lambda}_0) \right] + \frac{4(d + A_c)}{\lambda_c}.
	\end{equation}
	Therefore, by (\ref{L_c}) we obtain for $nT \le t < (n+1)T, n \in \mathbb{N}$
	$$ E_{\mathbf{z}}\left[ \mathcal{V}(\widehat{V}(t,nT,V^{\lambda}_{nT}),\widehat{X}(t,nT,V^{\lambda}_{nT})) \right] \le E_{\mathbf{z}}\left[ \mathcal{V}(V^{\lambda}_{nT},X^{\lambda}_{nT}) \right]  + \frac{d + A_c}{\lambda_c}.$$
	Then the processes in (\ref{proc_hat}) is uniformly bounded in $L^2$ by (\ref{ineq_lower_V}) and (\ref{ineq_bound_V_algo}),
	$$ \sup_{t \ge 0} E\left[\|\widehat{V}_t\|^2 \right] \le \frac{4}{(1- 2 \lambda_c) \beta } \left( \int_{\mathbb{R}^{2d}}{\mathcal{V}(x,v)d\mu_0(x,v)} + \frac{5(d + A_c)}{\lambda_c} \right) = C^c_v.$$
	and
	$$ \sup_{t \ge 0} E\left[\|\widehat{X}_t\|^2 \right] \le \frac{8 }{(1- 2 \lambda_c) \beta \gamma^2} \left( \int_{\mathbb{R}^{2d}}{\mathcal{V}(x,v)d\mu_0(x,v)} + \frac{5(d + A_c)}{\lambda_c} \right) = C^c_x.$$ 
	Similarly, from (\ref{L_d}) and (\ref{ineq_bound_V_algo}), we obtain for $nT \le k < (n+1)T, n\in \mathbb{N}$,
	$$ E_{\mathbf{z}}\left[ \mathcal{V}(\tilde{V}^{\lambda}_k,\tilde{X}^{\lambda}_k) \right] \le E_{\mathbf{z}}\left[ \mathcal{V}(V^{\lambda}_{0},X^{\lambda}_{0}) \right] + \frac{8(d + A_c)}{\lambda_c},$$
	and the upper bounds for $\sup_{k \in \mathbb{N}}E[\|\tilde{V}^{\lambda}_k\|^2], \sup_{k \in \mathbb{N}}E[\|\tilde{X}^{\lambda}_k\|^2]$ are $C^a_v, C^a_x$, respectively.
\end{proof}
\begin{lemma}\label{lemma:rho_to_w}
	Let $\mu, \nu$ be any two probability measures on $\mathbb{R}^{2d}$. It holds that
	$$\mathcal{W}_{\rho}(\mu,\nu) \le c_{17} \left( 1 + \varepsilon_c \left( \int{\mathcal{V}^2d\mu}\right) ^{1/2} + \varepsilon_c \left( \int{\mathcal{V}^2d\nu}\right) ^{1/2} \right) \mathcal{W}_2(\mu,\nu),$$
\end{lemma}
where $c_{17}:=3\max\{ 1+\alpha_c, \gamma^{-1}\}$.
\begin{proof}
	From (2.11) of \cite{ear2}, we have that $h(x) \le x$, for $x \ge 0$, and from (\ref{eq:r}), $r((x_1,v_1),(x_2,v_2)) \le c_{17}/3\|(x_1,v_1)-(x_2,v_2)\|$. By definition (\ref{eq_dist_aux}), we estimate
	\begin{eqnarray*}
		\mathcal{W}_{\rho}(\mu,\nu)  &=& \inf_{\xi \in \Pi(\mu,\nu)} \int_{\mathbb{R}^{2d}}{\rho((x_1,v_1),(x_2,v_2)) \xi(d(x_1,v_1)d(x_2,v_2)) }\\
		&\le& \inf_{\xi \in \Pi(\mu,\nu)} \int_{\mathbb{R}^{2d}}{r((x_1,v_1),(x_2,v_2))\left(1+ \varepsilon_c \mathcal{V}(x_1,v_1) +  \varepsilon_c \mathcal{V}(x_2,v_2) \right) \xi(d(x_1,v_1)d(x_2,v_2)) }\\
		&\le& c_{17}/3 \inf_{\xi \in \Pi(\mu,\nu)} \int_{\mathbb{R}^{2d}}{\|(x_1,v_1) - (x_2,v_2)\| \left(1+ \varepsilon_c \mathcal{V}(x_1,v_1) +  \varepsilon_c \mathcal{V}(x_2,v_2) \right) \xi(d(x_1,v_1)d(x_2,v_2)) }\\
		&\le& c_{17} \left( 1 + \varepsilon_c \left( \int{\mathcal{V}^2d\mu}\right) ^{1/2} + \varepsilon_c \left( \int{\mathcal{V}^2d\nu}\right) ^{1/2} \right)  \mathcal{W}_2(\mu,\nu).
	\end{eqnarray*}
\end{proof}
\begin{lemma}\label{lemma:moment_bound}
	Let $1 \le q \in \mathbb{N}$. It holds that
	$$C^{2q}_V:= \sup_{k \in \mathbf{N}} E[\|V^{\lambda}_k\|^{2q}] < \infty, \qquad C^{2q}_X:= \sup_{k \in \mathbf{N}} E[\|X^{\lambda}_k\|^{2q}] < \infty.$$
\end{lemma}
\begin{proof}
	We will use the arguments in the proof of Lemma 12 of \cite{gao} to obtain the contraction for $\mathcal{V}(X^{\lambda}_k,V^{\lambda}_k)$ and in Lemma 3.9 of \cite{five} to obtain high moment estimates. First, we have
	\begin{eqnarray}
	F_{\mathbf{z}}(X^{\lambda}_{k+1}) - F_{\mathbf{z}}(X^{\lambda}_{k}) - \left\langle \nabla F_{\mathbf{z}}(X^{\lambda}_{k}), \lambda V^{\lambda}_k \right\rangle  &=& \int_0^1{\left\langle \nabla F_{\mathbf{z}}(X^{\lambda}_{k} + \tau \lambda V^{\lambda}_{k} ) - \nabla F_{\mathbf{z}}(X^{\lambda}_{k}), \lambda V^{\lambda}_{k}  \right\rangle } d\tau \nonumber \\
	&\le& \int_0^1{ \left\|\nabla F_{\mathbf{z}}(X^{\lambda}_{k} + \tau \lambda V^{\lambda}_{k} ) - \nabla F_{\mathbf{z}}(X^{\lambda}_{k}) \right\| \left\| \lambda V^{\lambda}_{k}  \right\|  d\tau  } \nonumber\\
	&\le&  \frac{1}{2}M \lambda^2 \|V^{\lambda}_k\|^2. \label{eq:F}
	\end{eqnarray}
	Denoting $\Delta^1_k = V^{\lambda}_{k} - \lambda[\gamma V^{\lambda}_{k} + g(X^{\lambda}_{k},U_{\mathbf{z},k})] $, we compute
	\begin{eqnarray}
	&& \|V^{\lambda}_{k+1}\|^2 \nonumber\\
	&=& \|\Delta^1_k\|^2 + 2 \gamma \beta^{-1} \lambda \| \xi_{k+1} \|^2 + 2 \sqrt{2 \gamma \beta^{-1} \lambda}\left\langle  \Delta^1_k, \xi_{k+1}\right\rangle  \nonumber \\
	&\le& \|V^{\lambda}_{k} - \lambda[\gamma V^{\lambda}_{k} + \nabla F_{\mathbf{z}} (X^{\lambda}_k) ]\|^2 + \lambda^2 \| g(X^{\lambda}_{k},U_{\mathbf{z},k}) - \nabla F_{\mathbf{z}} (X^{\lambda}_k) \|^2 + 2 \gamma \beta^{-1} \lambda \| \xi_{k+1} \|^2 + 2 \sqrt{2 \gamma \beta^{-1} \lambda} \left\langle \Delta^1_k, \xi_{k+1}\right\rangle  \nonumber\\
	&\le& (1-\lambda \gamma)^2  \|V^{\lambda}_{k}\|^2 - 2\lambda(1-\lambda \gamma)\left\langle  \nabla F_{\mathbf{z}} (X^{\lambda}_k), V^{\lambda}_k \right\rangle + \lambda^2 \|F_{\mathbf{z}} (X^{\lambda}_k)\|^2 + \lambda^2 \| g(X^{\lambda}_{k},U_{\mathbf{z},k}) - \nabla F_{\mathbf{z}} (X^{\lambda}_k) \|^2 \nonumber \\
	&& \qquad + 2 \gamma \beta^{-1} \lambda \| \xi_{k+1} \|^2 + 2 \sqrt{2 \gamma \beta^{-1} \lambda} \left\langle  \Delta^1_k, \xi_{k+1}\right\rangle . \nonumber\\
	&\le& (1-\lambda \gamma)^2  \|V^{\lambda}_{k}\|^2 - 2\lambda(1-\lambda \gamma)\left\langle  \nabla F_{\mathbf{z}} (X^{\lambda}_k), V^{\lambda}_k \right\rangle + 3\lambda^2 (M \|X^{\lambda}_k\| +B)^2  + 2 \gamma \beta^{-1} \lambda \| \xi_{k+1} \|^2 + 2 \sqrt{2 \gamma \beta^{-1} \lambda} \left\langle  \Delta^1_k, \xi_{k+1}\right\rangle . \nonumber\\
	\label{eq:V} 
	\end{eqnarray}
	Similarly, we have
	\begin{equation}\label{eq:X}
	\|X^{\lambda}_{k+1}\|^2 = \|X^{\lambda}_{k}\|^2 + 2 \lambda\left\langle X^{\lambda}_{k}, V^{\lambda}_{k} \right\rangle + \lambda^2 \|V^{\lambda}_{k}\|^2. 
	\end{equation}
	Denoting $\Delta^2_k = X^{\lambda}_{k} + \gamma^{-1} V^{\lambda}_{k} - \lambda \gamma^{-1} g(X^{\lambda}_{k},U_{\mathbf{z},k})$, we compute that
	\begin{eqnarray}
	&&\|X^{\lambda}_{k+1} + \gamma^{-1} V^{\lambda}_{k+1} \|^2 \nonumber\\
	&=& \| X^{\lambda}_{k} + \gamma^{-1} V^{\lambda}_{k} - \lambda \gamma^{-1} g(X^{\lambda}_{k},U_{\mathbf{z},k}) + \sqrt{2 \gamma^{-1} \beta^{-1} \lambda} \xi_{k+1} \|^2  \nonumber\\
	&=& \| X^{\lambda}_{k} + \gamma^{-1} V^{\lambda}_{k} - \lambda \gamma^{-1} g(X^{\lambda}_{k},U_{\mathbf{z},k})  \|^2 + 2 \gamma^{-1} \beta^{-1} \lambda  \|\xi_{k+1} \|^2 + 2  \sqrt{2 \gamma^{-1} \beta^{-1} \lambda}\left\langle  \Delta^2_k, \xi_{k+1}\right\rangle  \nonumber\\
	&\le& \| X^{\lambda}_{k} + \gamma^{-1} V^{\lambda}_{k} - \lambda \gamma^{-1} \nabla F_{\mathbf{z}}(X^{\lambda}_{k}) \|^2 + \lambda^2 \gamma^{-2} \| g(X^{\lambda}_{k},U_{\mathbf{z},k}) - F_{\mathbf{z}}(X^{\lambda}_{k}) \|^2 \nonumber \\
	&& \qquad + 2 \gamma^{-1} \beta^{-1} \lambda  \|\xi_{k+1} \|^2 + 2  \sqrt{2 \gamma^{-1} \beta^{-1} \lambda} \left\langle \Delta^2_k, \xi_{k+1}\right\rangle \nonumber\\
	&\le& \| X^{\lambda}_{k} + \gamma^{-1} V^{\lambda}_{k}\|^2 - 2 \lambda \gamma^{-1} \left\langle \nabla F_{\mathbf{z}}(X^{\lambda}_{k}),X^{\lambda}_{k} + \gamma^{-1} V^{\lambda}_{k}\right\rangle  + 3\lambda^2 \gamma^{-2}(M \|X^{\lambda}_{k} \| + B)^2\nonumber\\
	&& \qquad   + 2 \gamma^{-1} \beta^{-1} \lambda  \|\xi_{k+1} \|^2 + 2  \sqrt{2 \gamma^{-1} \beta^{-1} \lambda} \left\langle  \Delta^2_k, \xi_{k+1}\right\rangle . \label{eq:XV}
	\end{eqnarray}
	Let us denote $\mathcal{V}_{k} = \mathcal{V}(X^{\lambda}_k,V^{\lambda}_k)$. From (\ref{eq:F}), (\ref{eq:V}), (\ref{eq:X}) and (\ref{eq:XV}) we compute that 
	\begin{eqnarray}
	&& \frac{\mathcal{V}_{k+1} - \mathcal{V}_{k} }{\beta } \label{eq:VV}\\
	&\le&  \lambda \left\langle \nabla F_{\mathbf{z}}(X^{\lambda}_{k}), V^{\lambda}_k \right\rangle + \frac{1}{2}M \lambda^2 \|V^{\lambda}_k\|^2 \nonumber \\
	&-&  \frac{1}{2} \lambda  \gamma \left\langle \nabla F_{\mathbf{z}}(X^{\lambda}_{k}),X^{\lambda}_{k} + \gamma^{-1} V^{\lambda}_{k}\right\rangle  + \frac{3}{4} \lambda^2 (M \|X^{\lambda}_k\| +B)^2 + \nonumber\\
	&& \qquad +  \frac{1}{2}   \gamma \beta^{-1} \lambda  \|\xi_{k+1} \|^2 + \frac{1}{2}  \gamma^2 \sqrt{2 \gamma^{-1} \beta^{-1} \lambda} \left\langle  \Delta^2_k, \xi_{k+1}\right\rangle \nonumber\\
	&+& \frac{1}{4}(-2\lambda \gamma + \lambda^2 \gamma^2) \|V^{\lambda}_{k}\|^2 - \frac{1}{2}\lambda(1-\lambda \gamma)\left\langle  \nabla F_{\mathbf{z}} (X^{\lambda}_k), V^{\lambda}_k \right\rangle + \frac{3}{4} \lambda^2 (M \|X^{\lambda}_k\| +B)^2 + \nonumber \\
	&& \qquad + \frac{1}{2} \gamma \beta^{-1} \lambda \| \xi_{k+1} \|^2 + \frac{1}{2} \sqrt{2 \gamma \beta^{-1} \lambda} \left\langle \Delta^1_k, \xi_{k+1}\right\rangle \nonumber\\
	&-& \frac{1}{2} \lambda \gamma^2 \lambda_c \left\langle X^{\lambda}_{k}, V^{\lambda}_{k} \right\rangle - \frac{1}{4} \lambda^2 \gamma^2 \lambda_c    \|V^{\lambda}_{k}\|^2 \nonumber \\
	&=& - \frac{1}{2} \lambda \gamma \left\langle \nabla F_{\mathbf{z}}(X^{\lambda}_{k}),X^{\lambda}_{k}\right\rangle  - \frac{1}{2} \lambda \gamma \|V^{\lambda}_k\|^2 - \frac{1}{2} \lambda \gamma^2 \lambda_c \left\langle X^{\lambda}_{k}, V^{\lambda}_{k} \right\rangle + \lambda^2 \mathcal{E}_k \nonumber \\
	&& \qquad + \gamma \beta^{-1} \lambda  \|\xi_{k+1} \|^2 + \Sigma_k,   \nonumber
	\end{eqnarray}
	where 
	\begin{eqnarray*}
		\mathcal{E}_k&:=&\left( \frac{1}{2}M   + \frac{1}{4} \gamma^2 - \frac{1}{4} \gamma^2 \lambda_c\right) \|V^{\lambda}_k\|^2  +  \frac{3}{2} (M \|X^{\lambda}_k\| +B)^2 + \frac{1}{2}\gamma\left\langle  \nabla F_{\mathbf{z}} (X^{\lambda}_k), V^{\lambda}_k \right\rangle,\\
		\Sigma_k &:=& \frac{1}{2}  \gamma^2 \sqrt{2 \gamma^{-1} \beta^{-1} \lambda} \left\langle \Delta^2_k, \xi_{k+1}\right\rangle  + \frac{1}{2} \sqrt{2 \gamma \beta^{-1} \lambda} \left\langle \Delta^1_k, \xi_{k+1}\right\rangle . 
	\end{eqnarray*}
	Using the inequality (\ref{eq:drift}), we obtain
	\begin{eqnarray}
	\frac{\mathcal{V}_{k+1} - \mathcal{V}_{k} }{\beta} &\le& - \lambda \gamma \lambda_c F_{\mathbf{z}}(X^{\lambda}_k) - \frac{1}{4}\lambda \gamma^3 \lambda_c \|X^{\lambda}_k \|^2 + \lambda \gamma A_c/\beta - \frac{1}{2} \lambda \gamma \|V^{\lambda}_k\|^2 - \frac{1}{2} \lambda \gamma^2 \lambda_c \left\langle X^{\lambda}_{k}, V^{\lambda}_{k} \right\rangle + \lambda^2 \mathcal{E}_k \nonumber \\
	&& \qquad + \gamma \beta^{-1} \lambda  \|\xi_{k+1} \|^2 + \Sigma_k.   \label{eq:1}
	\end{eqnarray}
	The quantity $\mathcal{E}_k$ is bounded as follows
	\begin{eqnarray*}
		\mathcal{E}_k &\le& \left( \frac{1}{2}M   + \frac{1}{4} \gamma^2 - \frac{1}{4} \gamma^2 \lambda_c + \gamma \right) \|V^{\lambda}_k\|^2  + M^2(3+2\gamma)\|X^{\lambda}_k\|^2 + B^2 \left(  3+ 2\gamma \right) .  
	\end{eqnarray*}
	As in \cite{gao}, we deduce that
	\begin{eqnarray}
	\mathcal{V}_k/\beta &\ge& \max\left\lbrace \frac{1}{8}(1-2\lambda_c)\gamma^2\|X^{\lambda}_k\|^2, \frac{1}{4}(1-2\lambda_c)\|V^{\lambda}_k\|^2  \right\rbrace \label{eq:VXV} \\
	&\ge& \frac{1}{16}(1-2\lambda_c)\gamma^2\|X^{\lambda}_k\|^2 + \frac{1}{8}(1-2\lambda_c)\|V^{\lambda}_k\|^2. \nonumber
	\end{eqnarray}
	And then we get that 
	\begin{equation}\label{eq:E}
	\mathcal{E}_k \le K_1\mathcal{V}_k/\beta + K_2
	\end{equation}
	where 
	$$K_1 = \max\left\lbrace \frac{M^2(3+2\gamma)}{\frac{1}{16}(1-2\lambda_c)\gamma^2}, \frac{(M/2 + \gamma^2/4 - \gamma^2 \lambda_c /4 + \gamma)}{\frac{1}{8}(1-2\lambda_c)} \right\rbrace, \qquad K_2 = B^2(3 + 2\gamma).$$
	Similarly, we bound $\Sigma_k$, using (\ref{eq:VXV}) and the definitions of $\Delta^1_k, \Delta^2_k$,  
	\begin{eqnarray*}
		\|\Sigma_k\|^2 &\le&  2 \gamma^{3} \beta^{-1} \lambda \|\Delta^2_k\|^2 \|\xi_{k+1}\|^2 + 2 \gamma \beta^{-1} \lambda \|\Delta^1_k\|^2 \|\xi_{k+1}\|^2\\
		&\le& 2 \lambda  \gamma \beta^{-1} \|\xi_{k+1}\|^2 \left(  \gamma^{2}  \|X^{\lambda}_{k} + \gamma^{-1} V^{\lambda}_{k} - \lambda \gamma^{-1} g(X^{\lambda}_{k},U_{\mathbf{z},k})\|^2  +  \|V^{\lambda}_{k} - \lambda[\gamma V^{\lambda}_{k} + g(X^{\lambda}_{k},U_{\mathbf{z},k})]\|^2 \right) \\
		&\le& 2 \lambda  \gamma \beta^{-1} \|\xi_{k+1}\|^2 \left( 3 \gamma^2 \|X^{\lambda}_{k}\|^2 + 3\|V^{\lambda}_{k}\|^2 +3(M\|X^{\lambda}_{k}\| + B)^2 + 2(1-\lambda \gamma)^2 \|V^{\lambda}_{k}\|^2 + 2(M\|X^{\lambda}_{k}\| + B)^2   \right)\\
		&\le&  2 \lambda  \gamma \beta^{-1} \|\xi_{k+1}\|^2 \left( (3 \gamma^2 + 10M^2) \|X^{\lambda}_{k}\|^2 + (3 + 2(1-\lambda \gamma)^2)\|V^{\lambda}_{k}\|^2 + 10B^2   \right).
	\end{eqnarray*}
	and thus 
	\begin{equation}\label{eq:bound=sig}
	\|\Sigma_k\|^2 \le  \left( P_1 \mathcal{V}_k/\beta + P_2\right) \lambda \|\xi_{k+1}\|^2 
	\end{equation}
	where
	$$P_1 = 2\max \left\lbrace \frac{2 \gamma \beta^{-1} (3 \gamma^2 + 10M^2)}{\frac{1}{16}(1-2 \lambda_c) \gamma^2}, \frac{2  \gamma \beta^{-1}(3 + 2(1-\lambda \gamma)^2)}{\frac{1}{8}(1-2 \lambda_c) } \right\rbrace , \qquad P_2 = 20  \gamma \beta^{-1} B^2.$$
	Noting that $\lambda_c \le 1/4,$ we have
	\begin{eqnarray*}
		\mathcal{V}_k/\beta &=& F_{\mathbf{z}}(X^{\lambda}_k) + \frac{1}{4}\gamma^2(1-\lambda_c)\|X^{\lambda}_k\|^2 + \frac{1}{2}\gamma \left\langle X^{\lambda}_k, V^{\lambda}_k \right\rangle + \frac{1}{2} \|V^{\lambda}_k\|^2 \\
		&\le& F_{\mathbf{z}}(X^{\lambda}_k) + \frac{1}{4}\gamma^2\|X^{\lambda}_k\|^2 + \frac{1}{2}\gamma \left\langle X^{\lambda}_k, V^{\lambda}_k \right\rangle + \frac{1}{2\lambda_c} \|V^{\lambda}_k\|^2.
	\end{eqnarray*}
	From (\ref{eq:1}), (\ref{eq:X}) we obtain
	\begin{eqnarray*}
		\frac{\mathcal{V}_{k+1} - \mathcal{V}_{k} }{\beta} &\le& - \lambda \gamma \lambda_c \left( F_{\mathbf{z}}(X^{\lambda}_k) + \frac{1}{4} \gamma^2 \|X^{\lambda}_k \|^2 -   A_c/(\beta \lambda_c) + \frac{1}{2 \lambda_c}  \|V^{\lambda}_k\|^2 + \frac{1}{2} \gamma \left\langle X^{\lambda}_{k}, V^{\lambda}_{k} \right\rangle \right)  +  \nonumber \\
		&& \qquad + \lambda^2 \mathcal{E}_k +  \gamma \beta^{-1} \lambda  \|\xi_{k+1} \|^2 + \Sigma_k \\
		&\le& \lambda \gamma \left(A_c/\beta  - \lambda_c \mathcal{V}_k/\beta \right) + (K_1\mathcal{V}_{k}/\beta + K_2) \lambda^2 + \gamma \beta^{-1} \lambda  \|\xi_{k+1} \|^2 + \Sigma_k.    
	\end{eqnarray*}
	Therefore, for $0< \lambda < \frac{\gamma \lambda_c}{2K_1}$
	$$\mathcal{V}_{k+1} \le \phi \mathcal{V}_{k} + \tilde{K}_{k+1}$$
	where 
	\begin{equation} \label{eq:phi_K}
	\phi:= 1 - \lambda \gamma \lambda_c/2, \qquad \tilde{K}_{k+1}:= \lambda \gamma A_c + \lambda^2\beta K_2 + \lambda \gamma \|\xi_{k+1}\|^2 + \beta \Sigma_{k}.
	\end{equation}
	Define $E_k[\cdot] := E[\cdot|(X^{\lambda}_k,V^{\lambda}_k), \mathbf{Z} = \mathbf{z}]$. We then compute as follows,
	\begin{eqnarray}
	E_k[\mathcal{V}^{2q}_{k+1}] &\le& E_k\left[ \left(| \phi \mathcal{V}_{k}|^2 + 2\phi \mathcal{V}_{k}\tilde{K}_{k+1}  + |\tilde{K}_{k+1} |^2  \right)^q  \right] \nonumber\\
	&&  \qquad = | \phi \mathcal{V}_{k} |^{2q} + 2q | \phi \mathcal{V}_{k} |^{2(q-1)} E_k\left[ \phi \mathcal{V}_{k} \tilde{K}_{k+1} \right]  + \sum_{k=2}^{2q} C^k_{2q} E_k\left[ | \phi \mathcal{V}_{k}|^{2q-k} | \tilde{K}_{k+1}|^k \right] \nonumber\\  
	\label{eq:2p0}
	\end{eqnarray}
	where the last inequality is due to Lemma A.3 of \cite{five}. Denoting $c_{19}:= \gamma A_c + \beta K_2 + \gamma d,$ we continue 
	\begin{eqnarray}
	E_k[\mathcal{V}^{2q}_{k+1}] &\le& | \phi \mathcal{V}_{k} |^{2q} + 2\lambda c_{19}  q | \phi \mathcal{V}_{k} |^{2q-1}     + \sum_{\ell=0}^{2q-2} {2q \choose \ell + 2} E_k\left[ | \phi \mathcal{V}_{k} |^{2q-2 -\ell} |\tilde{K}_{k+1} |^{\ell} |\tilde{K}_{k+1} |^{2} \right] \nonumber\\
	&\le& | \phi \mathcal{V}_{k} |^{2q} + 2\lambda c_{19}  q | \phi \mathcal{V}_{k} |^{2q-1}  + {2q \choose 2} \sum_{\ell=0}^{2q-2} {2q-2 \choose \ell} C^{\ell }_{2q-2} E_k\left[| \phi \mathcal{V}_{k} |^{2q-2 -\ell} |\tilde{K}_{k+1} |^{\ell} |\tilde{K}_{k+1} |^{2} \right] \nonumber\\
	&\le& | \phi \mathcal{V}_{k} |^{2q} + 2\lambda c_{19}  q | \phi \mathcal{V}_{k} |^{2q-1}  +q(2q-1)E_k\left[(| \phi \mathcal{V}_{k} | +  |\tilde{K}_{k+1}|)^{2q-2} |\tilde{K}_{k+1} |^{2}  \right] \nonumber\\
	&\le& | \phi \mathcal{V}_{k} |^{2q} + 2 \lambda c_{19}  q | \phi \mathcal{V}_{k} |^{2q-1} + q(2q-1) 2^{2q-3} | \phi \mathcal{V}_{k} |^{2q-2} E_k[|\tilde{K}_{k+1} |^{2}] +q(2q-1) 2^{2q-3} E_k[|\tilde{K}_{k+1} |^{2q}]. \nonumber \\
	\label{eq:2p}
	\end{eqnarray}
	Clearly we have
	\begin{eqnarray*}
		E_k\|\tilde{K}_{k+1}\|^2 &\le& 3 \lambda (\gamma A_c + \beta K_2)^2 + 3 \lambda \gamma^2 E\|\xi_{k+1}\|^4 +3 \lambda \beta d P_1|\mathcal{V}_k| + 3 \lambda \beta^2 d P_2 ,\\
		E_k\|\tilde{K}_{k+1}\|^{2q} &\le& 2^{2q-1} \lambda E \left(\gamma A_c + \beta K_2 + \gamma \|\xi_{k+1}\|^2 +   \beta \sqrt{P_2}\|\xi_{k+1}\|\right)^{2q} + 2^{2q-1}\lambda \beta^q P^q_1 |\mathcal{V}_k|^q E\|\xi_{k+1}\|^{2q}.   
	\end{eqnarray*}
	Define
	$$\tilde{M}_1:= \max \left\lbrace \frac{ (\gamma A_c + \beta K_2)^2 +  \gamma^2 E\|\xi_{k+1}\|^4+  \beta^2 d P_2}{ \beta d P_1}, \frac{\left( E \left(\gamma A_c + \beta K_2 + \gamma \|\xi_{k+1}\|^2 +   \beta \sqrt{P_2}\|\xi_{k+1}\|\right)^{2q}\right)^{1/q} }{\beta P_1 E^{1/q}\|\xi_{k+1}\|^{2q}  } \right\rbrace.$$
	On $\left\lbrace  \mathcal{V}_k \ge \tilde{M}_1\right\rbrace $ we have
	\begin{eqnarray*}
		E_k\|\tilde{K}_{k+1}\|^2 &\le& 6\lambda \beta d P_1|\mathcal{V}_k|,\\
		E_k\|\tilde{K}_{k+1}\|^{2q} &\le& 2^{2q} \lambda \beta^q P^q_1 |\mathcal{V}_k|^q E\|\xi_{k+1}\|^{2q}.
	\end{eqnarray*}
	And thus
	\begin{eqnarray}
	E_k[\mathcal{V}^{2q}_{k+1}] &\le& \phi  |\mathcal{V}_{k} |^{2q} + 2 \lambda c_{19}  q | \mathcal{V}_{k} |^{2q-1} + 6\lambda  q(2q-1) 2^{2q-3} \beta dP_1 | \mathcal{V}_{k} |^{2q-1} + \lambda q(2q-1) 2^{4q-3} \beta^q P^q_1 E\|\xi_{k+1}\|^{2q} | \mathcal{V}_{k} |^{q} \nonumber \\
	&=& (1 - \lambda \gamma \lambda_c/4) \mathcal{V}^{2q}_{k} \nonumber \\
	&-& \lambda \gamma \lambda_c/12 \mathcal{V}^{2q}_{k} + 2 \lambda c_{19}  q |  \mathcal{V}_{k} |^{2q-1} \nonumber\\
	&-& \lambda \gamma \lambda_c/12  \mathcal{V}^{2q}_{k} + 6\lambda  q(2q-1) 2^{2q-3} \beta dP_1 | \mathcal{V}_{k} |^{2q-1} \nonumber \\
	&-& \lambda \gamma \lambda_c/12  \mathcal{V}^{2q}_{k} +  \lambda q(2q-1) 2^{4q-3}  \beta^q P^q_1 E\|\xi_{k+1}\|^{2q} | \mathcal{V}_{k} |^{q}. \nonumber \\
	\label{eq:2p2} 
	\end{eqnarray}
	If we choose 
	$$\tilde{M}:= \max \left\lbrace  \tilde{M}_1, \frac{24 c_{19}q}{\gamma \lambda_c}, \frac{72q(2q-1) 2^{2q-3} \beta dP_1}{\gamma \lambda_c}, \left(\frac{ 12q(2q-1)2^{4q-3}\beta^q P^q_1E\|\xi_{k+1}\|^{2q}}{\gamma \lambda_c}\right)^{1/q} \right\rbrace  $$
	then on $\{ V_k \ge \tilde{M}\}$, the second, the third and the fourth term in the RHS of (\ref{eq:2p2}) are bounded by $0$ and then  
	\begin{eqnarray*}
		E_k[\mathcal{V}^{2q}_{k+1}] &\le& (1 - \lambda \gamma \lambda_c/4) \mathcal{V}^{2q}_{k}.
	\end{eqnarray*}
	On $\{ \mathcal{V}_{k} < \tilde{M} \}$, we have
	\begin{eqnarray*}
		E_k[\mathcal{V}^{2q}_{k+1}] &\le& (1 - \lambda \gamma \lambda_c/4)\mathcal{V}^{2q}_{k} + \lambda \tilde{N},
	\end{eqnarray*}
	where $\tilde{N}=2 c_{19}  q \tilde{M}^{2q-1} + 6 q(2q-1) 2^{2q-3} \beta dP_1 \tilde{M}^{2q-1} +  q(2q-1) 2^{4q-3} \beta^q P^q_1 E\|\xi_{k+1}\|^{2q} \tilde{M}^{q} $. For sufficiently small $\lambda$, we get from these bounds
	$$ E[\mathcal{V}^{2q}_k] \le  (1 - \lambda \gamma \lambda_c/4)^k\mathcal{V}^{2q}_0 + \frac{4\tilde{N}}{\gamma \lambda_c}.$$
	The proof is complete by using (\ref{eq:VXV}).
\end{proof}
\subsection{Explicit dependence of constants on important parameters}
Similar to \cite{gao}, we choose $\mu_0$ in such a way that 
$$\int_{\mathbb{R}^{2d}}\mathcal{V}(x,v)d\mu_0(dx,dv) = \mathcal{O}(\beta), \qquad \int_{\mathbb{R}^{2d}}e^{\mathcal{V}(x,v)}d\mu_0(dx,dv) = \mathcal{O}(e^{\beta}).$$
Then we get $C^c_x = C^c_v = C^a_x = C^a_v = \mathcal{O}((\beta + d)/\beta).$ It follows that
$$c_2 = c_3 = c_7 = c_{16} = \mathcal{O}(\sqrt{(\beta + d)/\beta}).$$
It is checked that 
$$A_c = \mathcal{O}(\beta), \qquad  \alpha_c = \mathcal{O}(1), \qquad  \Lambda_c = \mathcal{O}(\beta + d), \qquad R_1 = \mathcal{O}(\sqrt{1 + d/\beta}),$$ and
\begin{eqnarray*}
	c_* &=& \mathcal{O}(\sqrt{\beta+d}e^{-\mathcal{O}(\beta+d)}), \\
	C_* &=&\mathcal{O}\left( e^{\Lambda_c/p} \left(  R_1^{p-3} \frac{d+\beta}{\beta c_*} \right)^{1/p} \right) = \mathcal{O}\left( \frac{(d+\beta)^{1/2 - 1/(2p)}}{\beta^{1/2 - 1/(2p)}} \frac{e^{2\Lambda_c/p}}{\Lambda_c^{1/(2p)}} \right) \\
	&=& \mathcal{O}\left( \frac{(d+\beta)^{1/2 - 1/(2p)}}{\beta^{1/2 - 1/(2p)}} \left(  \frac{e^{\Lambda_c}}{\Lambda_c^{1/2}} \right)^{2/p} \Lambda^{1/(2p)}_c \right) = \mathcal{O}\left( \frac{(d+\beta)^{1/2}}{\beta^{1/2 - 1/(2p)} c^{2/p}_*} \right).   
\end{eqnarray*}
The constant $c_*, C_*$ are $\mu_*$ and $C$ respectively in \cite{gao}. In addition, we check
\begin{eqnarray*}
	c_{19} = \mathcal{O}(d+\beta), \qquad \tilde{M} = \tilde{M}_1 = \mathcal{O}((d + \beta)^2/d),\\
	\tilde{N} = \mathcal{O}\left( \frac{(d+\beta)^{4q-1}}{d^{2q-1}} \right), \qquad c_{18} = \mathcal{O}((d+\beta)^{3/2}/d^{1/2})
\end{eqnarray*}
and hence
$$\tilde{C} = \mathcal{O}\left( \frac{(d+\beta)^{1/2 + 2/p}}{\beta^{1/2} d^{1/(2p)} }  \frac{e^{-c_*}}{c^{2/p}_*(1-e^{-c_*})}   \right).$$
From Lemma 16 of \cite{gao}, we get
$$\mathcal{W}_p (\mu_0,\pi_{\mathbf{z}}) = \mathcal{O}\left(\sqrt{\frac{\beta+d}{\beta}}\right). $$
Furthermore, it is observed that 
$$\sigma = \mathcal{O}\left( \frac{(d+\beta)^{1-1/(4q)}}{d^{1/2-1/(4q)}} \right) .$$ 
Therefore, for a fixed $k$, the term $\mathcal{B}_1$ is bounded by 
\begin{eqnarray*}
	\mathcal{B}_1 &=& \mathcal{O}\left( \frac{(d+\beta)^{3/2 + 2/p - 1/(4q)}}{\beta^{1/2} d^{1/2+1/(4q)} }  \frac{e^{-c_*}}{c^{2/p}_*(1-e^{-c_*})}    \right)(\lambda^{1/(2p)} + \delta^{1/(2p)}) \\
	&&+ \mathcal{O}\left( \frac{(d+\beta)^{3/2 + 1/(2p) - 1/(4q)}}{d^{1/2 - 1/(4q)}\beta^{1/2} c^{2/p}_*} \right) e^{-c_*k\lambda}.
\end{eqnarray*}
Since $c_*$ is exponentially small in $(\beta+d)$, our bound for $\mathcal{B}_1$ is worse than that of $\mathcal{J}_1(\varepsilon) + \overline{\mathcal{J}}_0(\varepsilon)$ given in \cite{gao}.

\section*{Acknowledgments}
Both authors were supported by the NKFIH (National Research, Development and Innovation Office, Hungary) grant KH 126505 and the ``Lend\"ulet'' grant LP 2015-6 of the Hungarian Academy of Sciences. The authors thank Minh-Ngoc Tran for helpful discussions.

\bibliography{references}
\end{document}